\documentclass[12pt]{article}
\usepackage{amsmath,amsfonts,amssymb,mathtools,amsthm}
\usepackage{graphicx}
\usepackage{float}
\usepackage{booktabs}
\usepackage[hidelinks]{hyperref}
\usepackage{tikz}
\usetikzlibrary{arrows.meta}
\usepackage{pgfplots}
\pgfplotsset{compat=1.18}
\usepackage{xcolor}
\definecolor{customBlack}{HTML}{000000}
\definecolor{customNavy}{HTML}{14213D}
\definecolor{customOrange}{HTML}{FCA311}
\definecolor{customGray}{HTML}{E5E5E5}
\definecolor{customWhite}{HTML}{FFFFFF}
\definecolor{custompink}{RGB}{181,76,83} 
\definecolor{customgray}{RGB}{211,211,211}
\definecolor{customblue}{RGB}{155,221,255}
\definecolor{customcharcol}{RGB}{63,62,62}
\definecolor{customgold}{RGB}{187,170,126}
\definecolor{customdarkgray}{RGB}{119,118,118}
\definecolor{customolive}{RGB}{99,142,61}
\definecolor{custommediumgray}{RGB}{235,235,235}
\definecolor{customblack}{RGB}{65,65,65}
\definecolor{surf}{RGB}{215,215,215}
\definecolor{plane}{RGB}{197,197,207}
\definecolor{linkred}{RGB}{233,86,94}
\definecolor{linkblue}{RGB}{0,37,167}

\usepackage{algorithm}
\usepackage{algpseudocode}

\theoremstyle{plain}
\newtheorem{theorem}{Theorem}[section]
\newtheorem{proposition}[theorem]{Proposition}

\theoremstyle{definition}

\theoremstyle{remark}
\newtheorem{remark}[theorem]{Remark}

\usepackage{natbib}
\usepackage{url} 

\newcommand{\blind}{0}

\newcommand{\x}{\tilde{x}}

\begin{document}

\def\spacingset#1{\renewcommand{\baselinestretch}%
{#1}\small\normalsize} \spacingset{1}

\if0\blind
{
  \title{\bf GeoERM: Geometry‑Aware Multi‑Task Representation Learning on Riemannian Manifolds}
  \date{}
  \author{Aoran Chen\\
    Department of Biostatistics, New York University \\
    and \\
    Yang Feng\\
   Department of Biostatistics, New York University}

  \maketitle
} \fi

\if1\blind
{
  \bigskip
  \bigskip
  \bigskip
  \begin{center}
    {\LARGE\bf Title}
\end{center}
  \medskip
} \fi

\bigskip
\begin{abstract}
Multi‑Task Learning (MTL) seeks to boost statistical power and learning efficiency by discovering structure shared across related tasks.  State‑of‑the‑art MTL representation methods, however, usually treat the latent representation matrix as a point in ordinary Euclidean space, ignoring its often non‑Euclidean geometry, thus sacrificing robustness when tasks are heterogeneous or even adversarial. We propose \emph{GeoERM}, a geometry‑aware MTL framework that embeds the shared representation on its natural Riemannian manifold and optimizes it via explicit manifold operations.  Each training cycle performs (i) a Riemannian gradient step that respects the intrinsic curvature of the search space, followed by (ii) an efficient polar retraction to remain on the manifold, guaranteeing geometric fidelity at every iteration.  The procedure applies to a broad class of matrix‑factorized MTL models and retains the same per‑iteration cost as Euclidean baselines. Across a set of synthetic experiments with task heterogeneity and on a wearable‑sensor activity‑recognition benchmark, GeoERM consistently improves estimation accuracy, reduces negative transfer, and remains stable under adversarial label noise, outperforming leading MTL and single‑task alternatives.
\end{abstract}

\vfill

\newpage
\spacingset{1.75} 

\section{Introduction}
Consider a high-resolution biomedical image, such as a 512 $\times$ 512 scan used for diagnostic purposes. Within this vast array of pixels, only certain low-dimensional representations capture meaningful information, as evidenced by the success of Convolutional Neural Networks (CNNs). Indeed, detecting subtle tissue abnormalities, inferring underlying pathologies, and predicting treatment responses rely on extracting meaningful patterns from complex, high-dimensional data.

Representation Learning (RL) addresses this problem by mapping raw inputs into low-dimensional embeddings that reveal the data’s most informative characteristics \citep{rostami2022transfer}. Over the past decade, RL has catalyzed advancements in fields such as computer vision, multilingual knowledge graph completion, and reinforcement learning \citep{gupta2017learning, chen2020multilingual}. Nevertheless, pre-trained embeddings—while generally effective—often prove unreliable when faced with limited or heterogeneous data, as well as in settings burdened by outliers \citep{qiao2017learning, qiao2018outliers, wang2018learning, raghu2019transfusion}.

Multi-Task Learning (MTL) extends these ideas by simultaneously learning representations for multiple, possibly related, tasks \citep{zhang2021survey}. By jointly modeling multiple objectives, MTL exploits shared structures and relationships between tasks to enhance each individual task’s performance \citep{baxter2000model, maurer2016benefit, du2020few, tripuraneni2020theory, tripuraneni2020provable, tian2023learning}. It refines the principles of RL by integrating structural assumptions that guide the learning process more effectively, offering a broader perspective that a single-task approach cannot match \citep{thekumparampil2021statistically, rostami2022transfer}. For example, models pre-trained on ImageNet and later adapted to specialized medical imaging tasks illustrate MTL’s value: knowledge acquired from general visual domains can boost performance in specific clinical applications \citep{denevi2020advantage, zhou2021multi, deng2022learning}.

Current MTL frameworks commonly leverage low-dimensional embeddings, employing structural constraints—such as sparse or group-sparse penalties—to extract shared patterns \citep{xu2021multitask, li2023targeting}. However, these approaches typically treat learned representations as parameters in high-dimensional Euclidean spaces, neglecting geometric properties inherent to orthogonality and manifold structures \citep{bastani2021predicting, li2022transfer, gu2022robust, duan2023adaptive, gu2023commute, tian2023learning, zheng2023sofari}. Overlooking these underlying geometries limits stability and adaptivity, particularly in heterogeneous or adversarial scenarios.

In this paper, we propose a geometry-aware MTL framework that learns low-dimensional representations shared across multiple tasks while explicitly accounting for the intrinsic manifold structure. By embedding orthogonality and manifold constraints directly into the learning process, rather than imposing them post hoc, our method achieves improved stability and robustness, even under challenging situations.

\paragraph{MTL framework}
Consider a multi-task learning setting with \( T \) prediction tasks. For each task \( t \in [T] \), we observe data \(\{(\boldsymbol{X}_i^{(t)}, y_i^{(t)})\}_{i=1}^n\), with predictors \(\boldsymbol{X}_i^{(t)} \in \mathbb{R}^p\) and responses \(y_i^{(t)}\). Each response is modeled as
\(y_i^{(t)} \sim P\bigl(y \mid \boldsymbol{X}_i^{(t)}; \boldsymbol{\beta}^{(t)*}\bigr),\)
where \(P(y \mid \boldsymbol{X}; \boldsymbol{\beta})\) denotes a task-specific predictive distribution parameterized by \(\boldsymbol{\beta} \in \mathbb{R}^p\).

In MTL, identifying a shared, low-dimensional structure that captures inter-task relationships is natural. A common approach assumes each task-specific parameter \(\boldsymbol{\beta}^{(t)*}\) lies within a subspace spanned by a small number of latent factors \citep{baxter2000model, maurer2016benefit, du2020few, tripuraneni2020theory, tripuraneni2020provable, thekumparampil2021statistically}. This commonly leads to a factorization:
\(\boldsymbol{\beta}^{(t)*} = \boldsymbol{A}^{(t)*}\boldsymbol{\theta}^{(t)*},\)
where \(\boldsymbol{\theta}^{(t)*} \in \mathbb{R}^r\) is a low-dimensional coefficient vector, and \(\boldsymbol{A}^{(t)*} \in \mathbb{R}^{p \times r}\) is a representation matrix encoding the alignment of task parameters within a shared subspace. Such factorizations enhance computational efficiency, regularization, and interpretability, leveraging the inherent structural benefits of multi-task learning \citep{deng2022learning, denevi2020advantage, zhou2021multi, xu2021multitask, li2023targeting}.

Despite these advantages, previous methods typically treat \(\boldsymbol{A}^{(t)*}\) as a parameter in Euclidean space, enforcing orthogonality or low-rank constraints post hoc. In truth, assuming \(\boldsymbol{A}^{(t)*}\) is orthonormal—i.e., \(\boldsymbol{A}^{(t)*^{\top}}\boldsymbol{A}^{(t)*} = \boldsymbol{I}_r\)—naturally places these representations on the Stiefel manifold. Such orthonormality is non-trivial since it embeds \(\boldsymbol{A}^{(t)*}\) within a curved geometric space. Neglecting this geometric structure and imposing constraints as an afterthought can produce suboptimal solutions.

\paragraph{Our Contribution: Incorporating Geometric Structure into MTL}

These shortcomings indicate a need to integrate additional structure directly into the optimization objective. We build upon established MTL formulations that factor each task’s parameter vector \(\boldsymbol{\beta}^{(t)*}\) into a low-dimensional component \(\boldsymbol{\theta}^{(t)*}\) and an orthonormal matrix \(\boldsymbol{A}^{(t)*}\). Unlike previous methods that treat these representation matrices as parameters in Euclidean space, we explicitly constrain each \(\boldsymbol{A}^{(t)*}\) to reside on the Stiefel manifold. By respecting the underlying geometric structure, we operate directly within the manifold framework, resulting in a more coherent and principled approach.

We thus ask: \emph{Can exploiting the intrinsic geometric structure of task representations improve parameter estimation while also producing more stable and robust multi-task learning?} The answer is \emph{yes}. 

To achieve this, we introduce \emph{GeoERM}, a new Geometric Empirical Risk Minimization framework that employs Riemannian optimization to operate directly on the Stiefel manifold. By explicitly incorporating manifold geometry, GeoERM accurately captures cross-task relationships and improves parameter estimation. It also enhances robustness and generalization under heterogeneous or adversarial conditions, protecting performance where traditional methods stumble.

\paragraph{Roadmap}  
The remainder of this paper is organized as follows. Section~\ref{method section} formalizes the problem setup and outlines our core approach, which leads to our main GeoERM algorithm (Section~\ref{Main Algorithm GeoERM}). Section~\ref{Numerical Experiments} evaluates GeoERM through numerical experiments, first comparing it with baseline methods (Section~\ref{Models for Comparison}), then assessing its performance on simulated (Section~\ref{Simulation}) and real-world data (Section~\ref{Realdata}). Finally, Section~\ref{discussion} discusses broader implications and future research directions.

\section{Geometric Multi-task Learning}\label{method section}
In this section, we develop a geometry-aware framework for multi-task learning (MTL). After formalizing the problem and introducing a decomposition that captures shared structure and outliers, we present the key geometric tools—Riemannian gradient via orthogonal projection (Section~\ref{Riemannian Gradient Computation via Orthogonal Projection}) and polar retraction (Section~\ref{Polar Retraction on the Stiefel Manifold}). Section~\ref{sec:workflow-phases} describes the optimization workflow, and Section~\ref{Main Algorithm GeoERM} integrates these components into the full GeoERM algorithm. 

\subsection{Problem Setup}
We consider a multi-task learning (MTL) scenario involving \( T \) supervised learning tasks. Each task \( t \in [T] \) is associated with \( n \) observed data points \(\{(\boldsymbol{X}_i^{(t)}, y_i^{(t)})\}_{i=1}^n\), where \(\boldsymbol{X}_i^{(t)} \in \mathbb{R}^p\) and \(y_i^{(t)} \in \mathbb{R}\). While most tasks share common underlying structures, some may deviate substantially from these patterns, effectively behaving as outliers.

\paragraph{Predictive Model with Task-Specific Structure}  
For each task, we model the conditional distribution of the response as  
\(
y_i^{(t)} \sim P\bigl(y \mid \boldsymbol{X}_i^{(t)}; \boldsymbol{\beta}^{(t)}\bigr),
\)
where \( P(y \mid \boldsymbol{X}; \boldsymbol{\beta}) \) is parameterized by a task-specific coefficient vector \( \boldsymbol{\beta} \in \mathbb{R}^p \). For regression tasks, we assume a standard linear model  
\(
y_i^{(t)} = \langle \boldsymbol{X}_i^{(t)}, \boldsymbol{\beta}^{(t)} \rangle + \epsilon_i^{(t)},
\)
where \( \epsilon_i^{(t)} \) is independent, zero-mean sub-Gaussian noise with variance \( \sigma^2 \). For binary classification tasks, we adopt a logistic regression model 
\(
P(y_i^{(t)} = 1 \mid \boldsymbol{X}_i^{(t)}; \boldsymbol{\beta}^{(t)}) = \frac{1}{1 + e^{-\langle \boldsymbol{X}_i^{(t)}, \boldsymbol{\beta}^{(t)} \rangle}}.
\)

\paragraph{Task Parameter Decomposition and Outlier Modeling}  
We consider a multi-task learning scenario where the division of tasks into normal and outlier subsets is unknown a priori. To balance shared structure and individual task variability, we assume that normal tasks follow a common low-dimensional structure, while outlier tasks deviate arbitrarily. 

For normal tasks \( t \in S \subseteq [T] \), we decompose task parameters as 
\( \boldsymbol{\beta}^{(t)} = \boldsymbol{A}^{(t)} \boldsymbol{\theta}^{(t)} \), where \( \boldsymbol{A}^{(t)} \in \operatorname{St}(p, r) = \{ \boldsymbol{A} \in \mathbb{R}^{p \times r} : \boldsymbol{A}^\top \boldsymbol{A} = \boldsymbol{I}_r \} \) is an orthonormal representation matrix, and \( \boldsymbol{\theta}^{(t)} \in \mathbb{R}^r \) is a low-dimensional parameter vector.

For outlier tasks \( t \in S^c \), task parameters may take arbitrary values: 
\( \boldsymbol{\beta}^{(t)} = \boldsymbol{\beta}_{\text{outlier}}^{(t)} \). 
This distinction prevents outlier tasks from distorting the learned manifold-based structure.

\paragraph{Objective Function}  
We minimize the average loss across tasks and introduce a penalty that anchors each representation matrix \( \boldsymbol{A}^{(t)} \) to a shared center representation \( \overline{\boldsymbol{A}} \) on the Stiefel manifold:  
\begin{equation}\label{Objective}
    \frac{1}{T}\sum_{t=1}^T f^{(t)}\bigl(\boldsymbol{A}^{(t)}\boldsymbol{\theta}^{(t)}\bigr) + \frac{\lambda}{\sqrt{n}} \bigl\|\boldsymbol{A}^{(t)}(\boldsymbol{A}^{(t)})^\top - \overline{\boldsymbol{A}}(\overline{\boldsymbol{A}})^\top\bigr\|_2,
\end{equation}
where \( f^{(t)}: \mathbb{R}^p \to \mathbb{R} \) is the task-specific loss function, which is given as follows:

\begin{itemize}
    \item Linear regression: \(
f^{(t)}\bigl(\boldsymbol{A}^{(t)}\boldsymbol{\theta}^{(t)}\bigr) = \frac{1}{2n} \bigl\|\boldsymbol{Y}^{(t)} - \boldsymbol{X}^{(t)} \boldsymbol{A}^{(t)}\boldsymbol{\theta}^{(t)}\bigr\|_2^2 = \frac{1}{2n} \sum_{i=1}^n \Bigl(y_i^{(t)} - \boldsymbol{X}_i^{(t)\top}\boldsymbol{A}^{(t)}\boldsymbol{\theta}^{(t)}\Bigr)^2.\) 
    \item Logistic regression: \(
f^{(t)}\bigl(\boldsymbol{A}^{(t)}\boldsymbol{\theta}^{(t)}\bigr) = \frac{1}{n} \sum_{i=1}^n \Bigl( - y_i^{(t)} \boldsymbol{X}_i^{(t)\top} \boldsymbol{A}^{(t)}\boldsymbol{\theta}^{(t)} + \log \big(1 + e^{\boldsymbol{X}_i^{(t)\top} \boldsymbol{A}^{(t)}\boldsymbol{\theta}^{(t)}} \big) \Bigr).
\) 
\end{itemize}

Here, \( \boldsymbol{X}^{(t)} \in \mathbb{R}^{n \times p} \) and \( \boldsymbol{Y}^{(t)} \in \mathbb{R}^n \) represent the feature matrix and response vector for task \( t \), respectively, with the logistic regression loss following the standard negative log-likelihood formulation.

\subsection{Geometric Optimization on the Stiefel Manifold}\label{Geometric Optimization on the Stiefel Manifold}

From our problem formulation, each normal task parameter vector \( \boldsymbol{\beta}^{(t)} \), for \( t \in S \), follows the decomposition  \(\boldsymbol{\beta}^{(t)} = \boldsymbol{A}^{(t)}\boldsymbol{\theta}^{(t)}\), where \(\boldsymbol{A}^{(t)} \in \mathbb{R}^{p \times r}\) and \(\boldsymbol{\theta}^{(t)} \in \mathbb{R}^r\). The constraint
\(
\boldsymbol{A}^{(t)\top}\boldsymbol{A}^{(t)} = \boldsymbol{I}_r
\)
places \( \boldsymbol{A}^{(t)} \) on the Stiefel manifold \( \operatorname{St}(p, r) \), governing the structure of task representations. This geometric structure captures essential relationships among tasks but also demands specialized methods, as standard Euclidean optimization cannot directly handle manifold constraints.

\paragraph{Challenges in Manifold-Constrained Optimization}  
Applying a standard Euclidean gradient descent update,
\(
\boldsymbol{A}_{k+1}^{(t)} = \boldsymbol{A}_k^{(t)} - \alpha \nabla_{\boldsymbol{A}_k^{(t)}} f,
\)
does not preserve orthonormality. To avoid confusion, we denote the standard Euclidean gradient by \(\nabla\) and the Riemannian gradient by \(\tilde{\nabla}\). Even if \(\boldsymbol{A}_k^{(t)}\) satisfies \(\boldsymbol{A}_k^{(t)\top}\boldsymbol{A}_k^{(t)} = \boldsymbol{I}_r\), the iterate \(\boldsymbol{A}_{k+1}^{(t)}\) typically will not. This drift off the manifold weakens the geometric structure central to our MTL framework.

Addressing this issue requires specialized optimization on the Stiefel manifold so each update remains on the manifold, preserving orthogonality and geometric consistency for robust multi-task learning.

\paragraph{Why Simple Orthogonalization Falls Short
}
Let \(f: \operatorname{St}(p,r) \to \mathbb{R}\) be a continuously differentiable cost function defined on the Stiefel manifold \(\operatorname{St}(p,r)\). Under standard conditions—smoothness, Lipschitz continuity of \(\tilde{\nabla} f\), and a sufficiently small step size \(\alpha>0\)—Riemannian gradient descent on \(\operatorname{St}(p,r)\) satisfies a descent-type inequality analogous to the Euclidean setting. Theorem 4.3.1 in \citet{absil2008optimization} states that for some constant \(c>0\):
\(
f(x_{k+1}) \leq f(x_k) - c \alpha \|\tilde{\nabla} f(x_k)\|^2.
\), where \(\{x_k\} \subset \operatorname{St}(p,r)\) denotes iterates from Riemannian gradient descent. The term \(\tilde{\nabla} f(x_k)\) represents the Riemannian gradient of \(f\) at \(x_k\), which lies in the tangent space \(T_{x_k} \operatorname{St}(p,r)\). Consequently, \(\{f(x_k)\}\) strictly decreases, and \(\|\tilde{\nabla} f(x_k)\|\) converges to zero. Thus, every accumulation point (that is, the limit of any convergent subsequence of iterates) must be first-order stationary, i.e., \(\tilde{\nabla} f = 0\). In simpler terms, once the gradient vanishes at a limit point, no further local decrease in \(f\) is possible.

In contrast, simply taking a Euclidean step and followed by orthogonalization gives \( x_{k+1}' = \Pi_{\operatorname{St}(p,r)}(x_k' - \alpha \nabla f(x_k')) \). We denote these naive iterates by \( x_k' \), distinct from the iterates \( x_k \) obtained via Riemannian gradient descent. Although \( x_k' \in \operatorname{St}(p,r) \) , the update procedure does not guarantee that \( x_{k+1}' - x_k' \) aligns with the tangent direction \( -\alpha \tilde{\nabla} f(x_k') \). The projection \( \Pi_{\operatorname{St}(p,r)} \) is a non-differentiable operator rather than a smooth retraction, potentially introducing arbitrary rotations or folds not generated by a proper tangent vector and retraction step. Consequently, this naive approach fails to satisfy Riemannian convergence conditions, emphasizing the need for a principled, geometry-aware optimization framework.

\subsubsection{Proposed Approach: Geometry-Aware Optimization}  
Let \(\mathcal{M} \coloneqq \operatorname{St}(p, r)\) be the Stiefel manifold. Our geometry-aware procedure enforces orthogonality constraints by operating directly on \(\mathcal{M}\). We denote by \(T_{\boldsymbol{A}^{(t)}} \mathcal{M}\) the tangent space of \(\mathcal{M}\) at \(\boldsymbol{A}^{(t)}\). 
The method proceeds in two key steps: computing the Riemannian gradient and applying a proper retraction operator.

\paragraph{Step 1: Riemannian Gradient}  
Starting with the Euclidean gradient \(\nabla_{\boldsymbol{A}^{(t)}} \bar{f}\), where \(\bar{f}\) extends the original objective to \(\mathbb{R}^{p \times r}\), we obtain the Riemannian gradient \(\tilde{\nabla}_{\boldsymbol{A}^{(t)}} f\) by projecting onto the tangent space:  
\( \tilde{\nabla}_{\boldsymbol{A}^{(t)}} f = \mathcal{P}_{T_{\boldsymbol{A}^{(t)}} \mathcal{M}}(\nabla_{\boldsymbol{A}^{(t)}} \bar{f}) \),  
where the projection operator 
\( \mathcal{P}_{T_{\boldsymbol{A}^{(t)}} \mathcal{M}}(\boldsymbol{G}) = \boldsymbol{G} - \boldsymbol{A}^{(t)}\operatorname{sym}((\boldsymbol{A}^{(t)})^\top\boldsymbol{G}) \), where \( \operatorname{sym}(\boldsymbol{X}) \) denotes the symmetrization operator, defined as  
\(\operatorname{sym}(\boldsymbol{X}) \coloneqq \tfrac{1}{2}(\boldsymbol{X} + \boldsymbol{X}^\top).\)  This projection ensures that the descent direction respects the manifold’s geometry.

\paragraph{Step 2: Retraction Operator}  
After taking a step along the Riemannian gradient, we apply the polar retraction to map the result back onto the manifold:  
\( \mathcal{R}_{\boldsymbol{A}^{(t)}}(\boldsymbol{H}) = (\boldsymbol{A}^{(t)} + \boldsymbol{H})(\boldsymbol{I}_r + \boldsymbol{H}^\top\boldsymbol{H})^{-1/2} \),  
where \( \boldsymbol{H} = -\alpha \tilde{\nabla}_{\boldsymbol{A}^{(t)}} f \). This operation ensures that the updated representation \( \boldsymbol{A}^{(t)} \) remains in \( \operatorname{St}(p, r) \).

By combining the Riemannian gradient projection with an appropriate retraction operator, our geometry-aware optimization approach preserves the intrinsic manifold structure at every iteration. We illustrate the two-step optimization process in Figure \ref{fig:riemannian_optimization}.

\begin{figure}[H]
\centering
\begin{tikzpicture}
    \begin{axis}[
        height=0.52\linewidth,      
        view={60}{30},
        axis lines=none,
        declare function={
            f(\x,\y)=10-(\x^2+\y^2);
        }
    ]       

    \addplot3[
        surf,
        faceted color=gray!60,
        fill=surf,
        opacity=0.9,
        samples=20,
        domain=-2.5:2.5,
        domain y=-2:2,
    ]{f(x,y)};
    
    \addplot3[
        fill=plane,
        opacity=1,
        fill opacity=0.45,
        surf,
        shader=flat,
        colormap={graymap}{color=(plane) color=(plane)},
        samples=2,  
        domain=-2.5:2.5,
        domain y=-2:2,
    ] {8.5};

    \addplot3[mark=*,black, mark size=1.2] coordinates {(0,0,8.5)}; 

    \addplot3[mark=*,customblack, mark size=1.2] coordinates {(1.2,1.2,7.12)};

    \addplot3[mark=*,black, mark size=1] coordinates {(1.2,1.2,10.5)};

    \addplot3[->, custompink, line width=1pt, arrows={-Latex[length=5pt, width=2.5pt]}] 
    coordinates {(0,0,8.5) (1.2,1.2,8.5)}; 

    \addplot3[->, black, line width=1pt, arrows={-Latex[length=5pt, width=2.5pt]}] coordinates {(0,0,8.5) (1.2,1.2,10.5)}; 

    \addplot3[->, customcharcol, line width=0.8pt, arrows={-Latex[length=5pt, width=2.5pt]}] coordinates {(1.2,1.2,10.5) (1.2,1.2,8.5)};
    
    \addplot3[-,dotted, customcharcol,line width=1 pt] coordinates {(1.2,1.2,8.5) (1.2,1.2,7.12)}; 

    \node[black, below left, scale=1] at (axis cs:0,0,8.5) {\(\boldsymbol{A}_{k}^{(t)}\)};

    \node[black, below left, scale=1] at (axis cs:1.2,1.9,6) {\(\boldsymbol{A}_{k+1}^{(t)}\)};

    \node[customdarkgray, below left, scale=1] at (axis cs:-2,-0.45,10) {\(T_{\boldsymbol{A}_{k}^{(t)}}\mathcal{M}\)};
    
    \node[customdarkgray, below, scale=1] at (axis cs:2,-1.4,6) {\(\mathcal{M}\)};
    
    \node[black, right, scale=0.8] at (axis cs:1.1,-0.4,12.3) {\(\nabla_{\boldsymbol{A}_{k}^{(t)}} \bar{f}\)};
        
    \node[custompink, right, scale=0.8] at (axis cs:1.5,-0.35,8.9) {\(\tilde{\nabla}_{\boldsymbol{A}_{k}^{(t)}} f\)};

    \node[customcharcol, below right, scale=0.8] at (axis cs:1.2,1.2,10.5) {Projection};

    \node[customcharcol, right, scale=0.8] at (axis cs:1.2,1.2,7.5) {Retraction};

\end{axis}
\end{tikzpicture}
\caption{A geometric illustration of the two-step optimization process on the Stiefel manifold $\mathcal{M}$. Starting from the \(k\)-th iteration point of the \(t\)-th task, \(\boldsymbol{A}_{k}^{(t)} \in \mathcal{M}\), the Euclidean gradient (black arrow) is first orthogonally projected onto the tangent space \(T_{\boldsymbol{A}_{k}^{(t)}}\mathcal{M}\) (light gray plane) to obtain the Riemannian gradient (red arrow). The subsequent retraction (dotted line) maps this gradient back onto the manifold, producing the updated point \(\boldsymbol{A}_{k+1}^{(t)} \in \mathcal{M}\). This process makes sure the updates \(\boldsymbol{A}_{k+1}^{(t)}\) remain on $\mathcal{M}$ at every iteration, thus preserving the geometric structure of the representation.}
\label{fig:riemannian_optimization}
\end{figure}
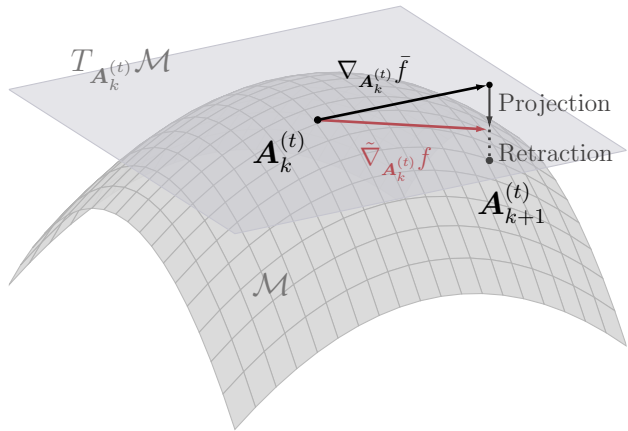

\subsection{Riemannian Gradient Computation via Orthogonal Projection}\label{Riemannian Gradient Computation via Orthogonal Projection}

We now compute the Riemannian gradient explicitly. The key idea is decomposing an arbitrary matrix \(\boldsymbol{G}\) into tangent and normal components at \(\boldsymbol{A}^{(t)}\), then removing the normal part. 

\subsubsection{Normal Space and Decomposition}
From the projection formula introduced earlier, computing the Riemannian gradient requires  
\(
\mathcal{P}_{T_{\boldsymbol{A}^{(t)}} \operatorname{St}(p, r)}(\boldsymbol{G}) = \boldsymbol{G} - N_{\boldsymbol{A}^{(t)}} \operatorname{St}(p, r),
\)
where \( \mathcal{P}_{T_{\boldsymbol{A}^{(t)}} \operatorname{St}(p, r)} \) is the orthogonal projection operator onto the tangent space of the Stiefel manifold at \( \boldsymbol{A}^{(t)} \). Here, \( N_{\boldsymbol{A}^{(t)}} \operatorname{St}(p, r) \) denotes the normal component, to be characterized explicitly.

Consider \( \boldsymbol{G} \in \mathbb{R}^{p \times r} \). We start by writing \( \boldsymbol{G} \) in terms of \( \boldsymbol{A}^{(t)} \) and its orthogonal complement \( (\boldsymbol{A}^{(t)})_\perp \):  
\( \boldsymbol{G} = \boldsymbol{A}^{(t)} \boldsymbol{W} + (\boldsymbol{A}^{(t)})_\perp \boldsymbol{K} \),  
where \( \boldsymbol{W} \in \mathbb{R}^{r \times r} \) and \( \boldsymbol{K} \in \mathbb{R}^{(p-r) \times r} \). The matrix \( (\boldsymbol{A}^{(t)})_\perp \) spans the orthogonal complement of \( \boldsymbol{A}^{(t)} \) in \( \mathbb{R}^p \).

Because each tangent vector \( \boldsymbol{H} \in T_{\boldsymbol{A}^{(t)}} \operatorname{St}(p, r) \) takes the form  
\( \boldsymbol{H} = \boldsymbol{A}^{(t)} \boldsymbol{\Omega} + (\boldsymbol{A}^{(t)})_\perp \boldsymbol{B} \),  
with \( \boldsymbol{\Omega} \in \operatorname{Skew}(r) \) and \( \boldsymbol{B} \in \mathbb{R}^{(p-r) \times r} \), we apply orthogonality conditions to find the normal space:  
\( \langle \boldsymbol{G}, \boldsymbol{H} \rangle = 0 \quad \forall \ \boldsymbol{\Omega} \in \operatorname{Skew}(r), \ \boldsymbol{B} \in \mathbb{R}^{(p-r) \times r},\) where \(\operatorname{Skew}(r)\) represents the set of all \(r \times r\) skew-symmetric matrices, i.e., matrices \(\boldsymbol{\Omega}\) satisfying \(\boldsymbol{\Omega}^T = -\boldsymbol{\Omega}\). For more details, see Section \ref{Skew-Symmetric Matrices}.

Orthogonality conditions imply
\(
\langle \boldsymbol{W}, \boldsymbol{\Omega} \rangle = 0, \quad \langle \boldsymbol{K}, \boldsymbol{B} \rangle = 0.
\) 
Let \( \operatorname{Sym}(r) \) denote the space of symmetric \( r \times r \) matrices. Since \( \operatorname{Skew}(r) \) and \( \operatorname{Sym}(r) \) are orthogonal complements, the condition \( \langle \boldsymbol{W}, \boldsymbol{\Omega} \rangle = 0 \) for all \( \boldsymbol{\Omega} \in \operatorname{Skew}(r) \) implies that \( \boldsymbol{W} \) must be symmetric. Thus, the normal space at \( \boldsymbol{A}^{(t)} \) is
\(
N_{\boldsymbol{A}^{(t)}} \operatorname{St}(p, r) = \{ \boldsymbol{A}^{(t)} \boldsymbol{S} : \boldsymbol{S} \in \operatorname{Sym}(r) \}
\).
Consequently, the tangent space is
\(
T_{\boldsymbol{A}^{(t)}} \operatorname{St}(p, r) = \{ \boldsymbol{A}^{(t)} \boldsymbol{\Omega} + (\boldsymbol{A}^{(t)})_\perp \boldsymbol{B} : \boldsymbol{\Omega} \in \operatorname{Skew}(r), \boldsymbol{B} \in \mathbb{R}^{(p-r) \times r} \}.
\) 

Given that the normal component at \( \boldsymbol{A}^{(t)} \) is \( \boldsymbol{A}^{(t)} \operatorname{sym}((\boldsymbol{A}^{(t)})^\top \boldsymbol{G}) \), we obtain:
\begin{equation}\label{Orthogonal_Projection}
\mathcal{P}_{T_{\boldsymbol{A}^{(t)}} \operatorname{St}(p, r)}(\boldsymbol{G}) = \boldsymbol{G} - \boldsymbol{A}^{(t)} \operatorname{sym}((\boldsymbol{A}^{(t)})^\top \boldsymbol{G}).
\end{equation}
Removing the symmetric part \( \operatorname{sym}((\boldsymbol{A}^{(t)})^\top \boldsymbol{G}) \) ensures the final vector stays in the tangent space.

\subsubsection{Riemannian Gradient Computation}  
For a smooth function \( f: \operatorname{St}(p, r) \to \mathbb{R} \) with a smooth Euclidean extension \(\bar{f}: \mathbb{R}^{p \times r} \to \mathbb{R}\), we start with \(\nabla_{\boldsymbol{A}^{(t)}} \bar{f}\). The Riemannian gradient \(\tilde{\nabla}_{\boldsymbol{A}^{(t)}} f\) is then:  
\begin{equation}
\tilde{\nabla}_{\boldsymbol{A}^{(t)}} f = \nabla_{\boldsymbol{A}^{(t)}} \bar{f} - \boldsymbol{A}^{(t)} \operatorname{sym}\left((\boldsymbol{A}^{(t)})^\top \nabla_{\boldsymbol{A}^{(t)}} \bar{f}\right).
\label{eq:riemannian_gradient}
\end{equation}

These geometry-aware steps are essential for preserving the orthonormal structure of \(\boldsymbol{A}^{(t)}\) during training. They support the geometric core of our multi-task learning approach. In practice, implementing this simply involves replacing \(\nabla_{\boldsymbol{A}^{(t)}} \bar{f}\) by \(\tilde{\nabla}_{\boldsymbol{A}^{(t)}} f\) in the update.

\subsection{Polar Retraction on the Stiefel Manifold}\label{Polar Retraction on the Stiefel Manifold}
After computing a Riemannian gradient step, we must ensure the updated representation matrices stay on the Stiefel manifold. Retractions handle this by smoothly mapping points back onto the manifold. Among possible choices, the polar retraction performs particularly well on the Stiefel manifold due to its numerical stability and inherent preservation of orthonormality.

\subsubsection{Definition of Polar Retraction}
For a point \( \boldsymbol{A}^{(t)} \in \operatorname{St}(p, r) \) and tangent vector \( \boldsymbol{H} \in T_{\boldsymbol{A}^{(t)}} \operatorname{St}(p, r) \), the polar retraction \( \mathcal{R}_{\boldsymbol{A}^{(t)}} \) is defined as  
\( \mathcal{R}_{\boldsymbol{A}^{(t)}}(\boldsymbol{H}) = (\boldsymbol{A}^{(t)} + \boldsymbol{H})((\boldsymbol{A}^{(t)} + \boldsymbol{H})^\top (\boldsymbol{A}^{(t)} + \boldsymbol{H}))^{-1/2} \).  
Using the Gram matrix, this reduces to:
\begin{equation}
    \mathcal{R}_{\boldsymbol{A}^{(t)}}(\boldsymbol{H}) = (\boldsymbol{A}^{(t)} + \boldsymbol{H})(\boldsymbol{I}_r + \boldsymbol{H}^\top \boldsymbol{H})^{-1/2}. \label{eq:retraction}
\end{equation}

This leads to the iterative update  
\(
\boldsymbol{A}_{k+1}^{(t)} = \mathcal{R}_{\boldsymbol{A}_{k}^{(t)}}(\boldsymbol{H}),
\)
where \( \boldsymbol{H} = -\alpha \tilde{\nabla}_{\boldsymbol{A}^{(t)}} f \) is the negative Riemannian gradient direction and \( k \) is the iteration index. Applying the polar retraction keeps each updated \( \boldsymbol{A}_{k+1}^{(t)} \) on the manifold, preserving orthonormality central to our geometric approach.

Having established the core update mechanism, we now outline GeoERM's full workflow, detailing how these updates integrate into the overall learning procedure.

\subsection{Workflow and Algorithm Phases}\label{sec:workflow-phases}  
GeoERM proceeds in two steps: manifold optimization and parameter refinement. 

\paragraph{Step 1: Manifold-Constrained Updates}  
The first step updates the task-specific representation matrices \( \{\boldsymbol{A}^{(t)}\}_{t=1}^T \) and the shared center \( \overline{\boldsymbol{A}} \) by optimizing the objective:
\(
\sum_{t=1}^T \left[ f^{(t)}\left(\boldsymbol{A}^{(t)} \boldsymbol{\theta}^{(t)}\right) + \frac{\lambda}{\sqrt{n}} \left\| \boldsymbol{A}^{(t)} (\boldsymbol{A}^{(t)})^\top - \overline{\boldsymbol{A}} (\overline{\boldsymbol{A}})^\top \right\|_2 \right],
\)
encouraging each \( \boldsymbol{A}^{(t)} \) to remain close to a shared subspace while maintaining orthonormality.

The gradients of \( \{\boldsymbol{A}^{(t)}\}_{t=1}^T \) are computed jointly via backpropagation. We then project each \( \nabla_{\boldsymbol{A}^{(t)}} \bar{f} \) onto the tangent space of the Stiefel manifold:  
\(
\tilde{\nabla}_{\boldsymbol{A}^{(t)}} f = \mathcal{P}_{T_{\boldsymbol{A}^{(t)}} \operatorname{St}(p, r)}(\nabla_{\boldsymbol{A}^{(t)}} \bar{f}).
\)
This projection turns the Euclidean gradient into a Riemannian gradient. After applying Adam to update all parameters, we perform polar retraction to ensure each iterate remains on the manifold:  
\(
\boldsymbol{A}^{(t)} \leftarrow \mathcal{R}_{\boldsymbol{A}^{(t)}}(-\alpha \cdot \tilde{\nabla}_{\boldsymbol{A}^{(t)}} f).
\) A similar update occurs for \( \overline{\boldsymbol{A}} \), ensuring alignment with task-specific representations. 

\paragraph{Step 2: Parameter Refinement}  
In the second step, we update the task-specific regression coefficients \( \{\boldsymbol{\beta}^{(t)}\}_{t=1}^T \) using the learned low-rank representations from Step 1. For each task \( t \), we solve the following regularized problem
\(
\widehat{\boldsymbol{\beta}}^{(t)} = \underset{\boldsymbol{\beta} \in \mathbb{R}^p}{\arg \min} \left\{ f^{(t)}(\boldsymbol{\beta}) + \frac{\gamma}{\sqrt{n}} \left\| \boldsymbol{\beta} - \widehat{\boldsymbol{A}}^{(t)} \widehat{\boldsymbol{\theta}}^{(t)} \right\|_2 \right\}.
\)

This step refines each task's coefficients by regularizing toward the shared low-dimensional structure, while preserving task-specific variations through flexible shrinkage.

\subsection{Main Algorithm GeoERM}\label{Main Algorithm GeoERM}

\paragraph{Review of Prior Work}  
The GeoERM algorithm builds upon the two-step framework proposed by \citet{tian2023learning}, which splits multi-task learning into two linked phases. In the first phase, following \citet{duan2023adaptive}, a penalty \(\| \boldsymbol{A}^{(t)} (\boldsymbol{A}^{(t)})^\top - \overline{\boldsymbol{A}} (\overline{\boldsymbol{A}})^\top \|_2\) encourages each \(\boldsymbol{A}^{(t)}\) to align with a shared subspace while keeping its own features. In the second phase, regularization techniques taken from distance-based MTL and transfer learning \citep{scholkopf2001generalized,kuzborskij2013stability,kuzborskij2017fast} refine the task-specific parameters to reduce negative transfer. While GeoERM follows the conceptual structure of \citet{tian2023learning}, it distinguishes itself by using geometry-aware optimization on the Stiefel manifold. By optimizing directly on the manifold and using Riemannian gradients and retractions, GeoERM keeps the learned representation matrices orthonormal and respects the manifold’s intrinsic geometry.

{\linespread{1.2}
\begin{algorithm}[H]
\caption{GeoERM (Geometric Empirical Risk Minimizer)} \label{alg:geoerm}
\begin{algorithmic}[1]

\State \textbf{Input:} $\{\boldsymbol{X}^{(t)}, \boldsymbol{Y}^{(t)}\}_{t=1}^T = \{\{\boldsymbol{x}_i^{(t)}, y_i^{(t)}\}_{i=1}^n\}_{t=1}^T$, penalty parameters $\lambda$, $\gamma$, \texttt{num\_iterations}

\State \textbf{Output:} Estimators $\{\widehat{\boldsymbol{\beta}}^{(t)}\}_{t=1}^T$, $\widehat{\overline{\boldsymbol{A}}}$


\State \textbf{Step 1.} Initialize $\{\boldsymbol{A}^{(t)}\}_{t=1}^T \subset \mathcal{O}_{p \times r}$, $\overline{\boldsymbol{A}} \in \mathcal{O}_{p \times r}$, $\{\boldsymbol{\theta}^{(t)}\}_{t=1}^T \subset \mathbb{R}^r$

\For{iteration $= 1$ to \texttt{num\_iterations}}

    \State Compute Euclidean gradients of $\bar{f}$ (Eq.~\eqref{Objective}) with respect to all parameters

  \For{$t = 1, \dots, T$}
    \State Compute Riemannian gradient: $\tilde{\nabla}_{\boldsymbol{A}^{(t)}} f = \mathcal{P}_{T_{\boldsymbol{A}^{(t)}} \operatorname{St}(p, r)}(\nabla_{\boldsymbol{A}^{(t)}} \bar{f})$ \Comment{Eq.~\eqref{eq:riemannian_gradient}}
    \State Update via retraction: $\boldsymbol{A}^{(t)} \leftarrow \mathcal{R}_{\boldsymbol{A}^{(t)}}(-\alpha \cdot \tilde{\nabla}_{\boldsymbol{A}^{(t)}} f)$ \Comment{Eq.~\eqref{eq:retraction}}
    \State Update: $\boldsymbol{\theta}^{(t)} \leftarrow \boldsymbol{\theta}^{(t)} - \alpha \cdot \nabla_{\boldsymbol{\theta}^{(t)}} \bar{f}$
  \EndFor

  \State Compute Riemannian gradient:  $\tilde{\nabla}_{\overline{\boldsymbol{A}}} f = \mathcal{P}_{T_{\overline{\boldsymbol{A}}} \operatorname{St}(p, r)}(\nabla_{\overline{\boldsymbol{A}}} \bar{f})$ \Comment{Eq.~\eqref{eq:riemannian_gradient}}
  \State Update via retraction: $\overline{\boldsymbol{A}} \leftarrow \mathcal{R}_{\overline{\boldsymbol{A}}}(-\alpha \cdot \tilde{\nabla}_{\overline{\boldsymbol{A}}} f)$ \Comment{Eq.~\eqref{eq:retraction}}

\EndFor

\State \textbf{Step 1 output:} $\{\widehat{\boldsymbol{A}}^{(t)}, \widehat{\boldsymbol{\theta}}^{(t)}\}_{t=1}^T$, $\widehat{\overline{\boldsymbol{A}}}$


\State \textbf{Step 2.} Update $\{{\boldsymbol{\beta}}^{(t)}\}_{t=1}^T \subset \mathbb{R}^{p}$
\For{$t = 1, \dots, T$,} 
\State \(\widehat{\boldsymbol{\beta}}^{(t)} = \arg \min_{\boldsymbol{\beta} \in \mathbb{R}^p} \left\{ f^{(t)}(\boldsymbol{\beta}) + \frac{\gamma}{\sqrt{n}} \left\| \boldsymbol{\beta} - \widehat{\boldsymbol{A}}^{(t)} \widehat{\boldsymbol{\theta}}^{(t)} \right\|_2 \right\}\)
\EndFor
\end{algorithmic}
\end{algorithm}
}

\subsubsection{Inputs and Outputs} \label{sec:inputs-outputs}  
GeoERM takes as input \( T \) tasks each with \( n \) observations \(\{\boldsymbol{X}^{(t)}, \boldsymbol{Y}^{(t)}\}_{t=1}^T\). For each task \( t \), \(\boldsymbol{X}^{(t)} \in \mathbb{R}^{n \times p}\) are the input features and \(\boldsymbol{Y}^{(t)} \in \mathbb{R}^n\) are the outputs. The algorithm also needs regularization parameters \(\lambda\) and \(\gamma\): \(\lambda\) controls how closely each \(\boldsymbol{A}^{(t)}\) matches a shared center matrix \(\overline{\boldsymbol{A}}\), while \(\gamma\) steers task-specific coefficients toward their low-rank representations. Additional hyperparameters, such as the step size \(\alpha\) and the number of iterations, are also provided.

GeoERM outputs task-specific representation matrices \(\{\widehat{\boldsymbol{A}}^{(t)}\}_{t=1}^T\), each on the Stiefel manifold \(\operatorname{St}(p, r)\), and a shared center matrix \(\widehat{\overline{\boldsymbol{A}}}\) that captures common structure across tasks. It also returns refined task-specific regression coefficients \(\{\widehat{\boldsymbol{\beta}}^{(t)}\}_{t=1}^T\), balancing shared low-rank structures and task-specific flexibility.

\section{Numerical Experiments} \label{Numerical Experiments}

We conducted numerical experiments to evaluate GeoERM’s performance under different conditions. Below, we outline our data generation methods, evaluation metrics, and computational setup. All implementation code is available at \url{https://github.com/samohtaerg/GeoERM}.

\subsection{Models for Comparison}\label{Models for Comparison}

We compare GeoERM against several baseline and related methods, following the experimental setup from \citet{tian2023learning}, ensuring fair and consistent assessment. All methods use publicly available code with parameters as originally reported.

\begin{itemize}
    \item \textbf{GeoERM}: Our \textbf{GeoERM} algorithm integrates gradient projection, polar retraction, and loss minimization. Following \citet{tian2023learning}, we set \(\lambda = \sqrt{r(p + \log T)}\) and \(\gamma = \sqrt{p + \log T}\) to match pERM’s configuration. We employ the Adam optimizer \citep{kingma2015method} with a 0.01 learning rate and default hyperparameters.

    \item \textbf{Penalized Empirical Risk Minimization (pERM)}: Implemented from \citet{tian2023learning}, using their public code (\url{https://github.com/ytstat/RL-MTL-TL}) and a fixed learning rate of 0.01 with default settings.

    \item \textbf{Spectral Method}: As described by \citet{tian2023learning}, utilizing their implementation and Adam with a learning rate of 0.01.

    \item \textbf{Method-of-Moments (MoM)}: Implemented following \citet{tripuraneni2020provable} with default parameters.

    \item \textbf{Adaptive Representation Learning (AdaptRep)}: Using the code from \citet{chua2021fine}, retaining default parameters.

    \item \textbf{Adaptive and Robust Multi-Task Learning (ARMUL)}: Implemented according to \citet{duan2023adaptive} (\url{https://github.com/kw2934/ARMUL}) with default settings.

    \item \textbf{Pooled Regression}: Following \citet{tian2023learning}, based on methods by \citet{ben2010theory} and \citet{crammer2008learning}, using implementations from \texttt{sklearn.linear\_model} for linear and logistic regression.

    \item \textbf{Single-Task Regression}: Fits individual regression models separately for each task, serving as a baseline.
\end{itemize}

\subsection{Simulation}\label{Simulation}

\paragraph{Problem Setup} We generated synthetic datasets to test GeoERM. We consider \(T = 50\) tasks, each with \(n = 100\) observations and \(p = 50\) features. The latent representation dimension is \(r = 5\), and we vary \(h \in [0.1, 0.9]\) to control how much tasks differ. Coefficients come from $\text{Uniform}(-H, H)$ with \(H = 2\).

\paragraph{Regular tasks} We build a shared basis \( \boldsymbol{A}_{\text{center}} \in \mathbb{R}^{p \times r} \) from the top \( r \) singular vectors of a random Gaussian matrix. For regular tasks (\( t \in S \)), we add perturbations \( \Delta \boldsymbol{A} \sim \text{Uniform}(-h, h) \) and set  \(\boldsymbol{A}^{(t)} = \boldsymbol{A}_{\text{center}} + \Delta \boldsymbol{A}\). We then orthonormalize \(\boldsymbol{A}^{(t)}\) via a QR decomposition. Task-specific coefficients are \(\boldsymbol{\beta}^{(t)} = \boldsymbol{A}^{(t)} \boldsymbol{\theta}^{(t)}\) where \( \boldsymbol{\theta}^{(t)} \sim \text{Uniform}(-H, H) \).  

\paragraph{Outlier tasks} For outlier tasks \( t \in S^c \), \( \boldsymbol{\beta}^{(t)} \) is drawn independently from \( \text{Uniform}(-3, 3) \), ignoring the shared basis. Features \(\boldsymbol{X}^{(t)}\) follow \(\mathcal{N}(0,1)\) for regular tasks and \(\mathcal{N}(0,2)\) for outliers. Responses \(\boldsymbol{Y}^{(t)}\) come from a linear model \(\boldsymbol{Y}^{(t)} = \boldsymbol{X}^{(t)}\boldsymbol{\beta}^{(t)} + \boldsymbol{\epsilon}^{(t)}\) with \(\boldsymbol{\epsilon}^{(t)} \sim \mathcal{N}(0, \boldsymbol{I})\), or a logistic model using \(\sigma\left((\boldsymbol{X}^{(t)})^\top \beta^{(t)}\right)\).

\paragraph{Evaluation Metrics}\label{evalmetrics}

We evaluate performance using the maximum estimation error among non-outlier tasks $S \subseteq T$:
\(
\text{Error}_{\max} = \max_{t \in S} \|\hat{\boldsymbol{\beta}}^{(t)} - \boldsymbol{\beta}^{(t)}\|_2,
\)
where \(\hat{\boldsymbol{\beta}}^{(t)}\) are estimated coefficients and \(\boldsymbol{\beta}^{(t)}\) are the true values. This metric highlights performance on the most challenging tasks.

\paragraph{Computational Setup}\label{computationsetup}

All experiments were conducted in Python using PyTorch \citep{paszke2019pytorch} on NYU’s Greene HPC cluster. Each run employed one NVIDIA A100 GPU (32 GB) and four CPU cores. Results were averaged over 100 iterations to ensure stability.

\subsubsection{Simulation with Different Heterogeneity Parameter h}
We now examine how the heterogeneity parameter \(h\) influences model performance. As \(h\) increases, tasks become more diverse, enabling assessment of each method’s adaptability. We vary \(h\) under two scenarios: with no  outliers (\(\epsilon = 0\)) and with outliers (\(\epsilon = 0.1\)).

\begin{figure}[H]
    \centering
    \includegraphics[width=0.95\linewidth]{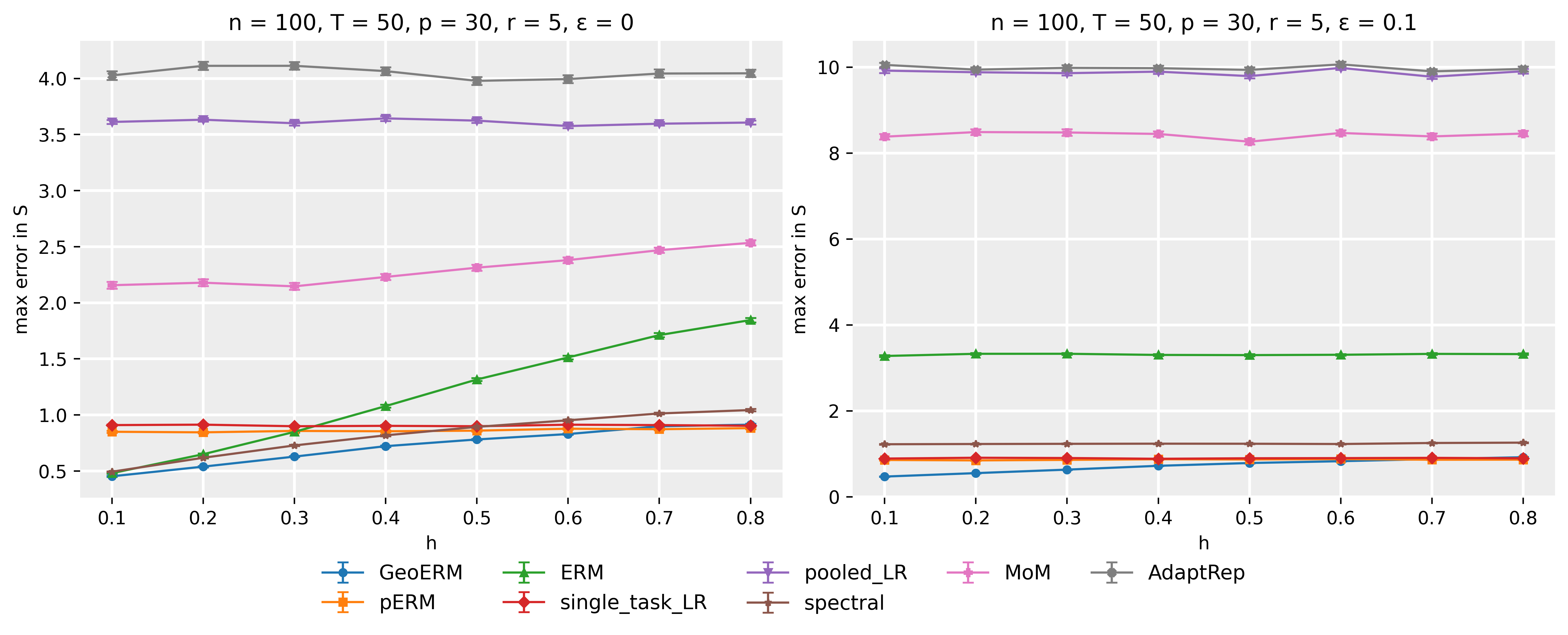}
    \caption{Maximum error across varying heterogeneity parameter \( h \), evaluated under two outlier proportion settings: \( \epsilon = 0 \) (left) and \( \epsilon = 0.1 \) (right). Simulations are conducted with \( n = 100 \), \( T = 50 \), \textbf{\( p = 30 \)}, and \( r = 5 \). Evaluation metrics and computational settings are described
 in Sections~\ref{evalmetrics}.}
    \label{fig:hpc305}
\end{figure}
Figure~\ref{fig:hpc305} shows the maximum error in the regular dataset \( S \) as we vary the heterogeneity level \( h \).
As \( h \) increases and tasks become more diverse, GeoERM consistently achieves the lowest error, demonstrating robustness to task heterogeneity. In the left panel (\(\epsilon = 0\)), pERM and AdaptRep exhibit stable performance across varying \( h \), whereas GeoERM shows a slight increase in error but maintains a clear advantage over all other methods at each \( h \). Spectral methods give moderate results, and pooled LR and single-task LR struggle with heterogeneous tasks. With outliers (\(\epsilon = 0.1\), right panel), GeoERM maintains minimal error growth as \( h \) increases, outperforming all other methods. MoM and spectral methods resist outliers to some extent but are less robust than GeoERM, while pooled LR suffers severely. Additional experimental results are provided in Appendix \ref{Supplementary Experiments}.

Overall, GeoERM consistently achieves the lowest estimation error and best adapts to changing \( h \), \( n \), and \( T \), as well as to varying levels of outlier contamination. Additional results for varying sample sizes \( n \) and task numbers \( T \) are provided in Appendix \ref{Supplementary Experiments}. The geometry-aware optimization provides stability and robustness under high-dimensional, heterogeneous conditions, outperforming other methods in adaptability and accuracy. While methods like pERM, spectral methods, and MoM show reasonable adaptability, none match GeoERM’s performance, highlighting its effectiveness in challenging multi-task environments.

\subsection{A Real-data Study}\label{Realdata}

To further evaluate GeoERM, we apply it to the Human Activity Recognition (HAR) dataset from the UCI Machine Learning Repository~\citep{anguita2013public}, accessible at \url{https://archive.ics.uci.edu/dataset/240/human+activity+recognition+using+smartphones}. The HAR dataset offers a challenging high-dimensional setting, naturally structured into tasks (subjects), making it ideal for assessing multi-task methods.

\subsubsection{Human Activity Recognition Dataset}

The HAR dataset contains recordings from 30 subjects performing six daily activities: \texttt{WALKING}, \texttt{WALKING\_UPSTAIRS}, \texttt{WALKING\_DOWNSTAIRS}, \texttt{SITTING}, \texttt{STANDING}, and \texttt{LAYING}. Each subject carried a Samsung Galaxy S II smartphone with inertial sensors sampling linear acceleration and angular velocity at 50Hz. The full dataset has \( n = 10{,}299 \) samples, \( T = 30 \) tasks (one per subject), and \( p = 561 \) time- and frequency-domain features. These features capture both common and subject-specific motion patterns, making this problem a natural fit for multi-task learning.

\subsubsection{Experimental Setup}

Following \citet{tian2023learning}, we normalize each subject’s feature matrix separately with a StandardScaler, so that they are comparable across tasks. We then group activities into two classes—static (\texttt{SITTING}, \texttt{LAYING}) and dynamic (all others)—resulting in a binary classification problem.

We evaluate multi-task logistic regression models, including GeoERM and several baselines: single-task regression, pooled regression, ERM, ARMUL, pERM, and spectral methods. For each model, we estimate task-specific parameters \(\hat{\beta}_t\) and measure classification error on a held-out test set. All implementations use default settings as described in Section~\ref{Models for Comparison}. Each experiment was repeated 100 times on the NYU Greene HPC cluster to ensure stable results.

We use classification error as our main metric, defined as the proportion of misclassified labels. By comparing errors across all tasks, we assess how well each method uses a shared structure while handling task differences.

\subsubsection{Results and Analysis}

Table~\ref{tab:results} reports the average classification error rates and standard deviations over 30 tasks for different dimensions of the representation space \( r \). GeoERM consistently achieves the lowest error, outperforming all other methods at every tested \( r \).

\begin{table}[H]
\centering
\caption{Classification error rates (mean and standard deviation) in percentage on test data, averaged across 30 tasks and varying \( r \).}
\label{tab:results}
\resizebox{\textwidth}{!}{ 
\begin{tabular}{lccccccc}
\toprule
$r$/Method & Single-task & Pooled & ERM & ARMUL & pERM & Spectral & GeoERM \\
\midrule
$r = 5$ & 1.68 (0.20) & 1.80 (0.17) & 1.50 (0.18) & 2.16 (0.22) & 1.35 (0.18) & 1.88 (0.19) & \textbf{1.04 (0.15)} \\
$r = 10$ & 1.67 (0.20) & 1.80 (0.18) & 1.40 (0.16) & 1.77 (0.21) & 1.33 (0.19) & 1.49 (0.18) & \textbf{1.02 (0.15)} \\
$r = 15$ & 1.69 (0.21) & 1.81 (0.18) & 1.39 (0.17) & 1.68 (0.21) & 1.34 (0.18) & 1.53 (0.20) & \textbf{1.04 (0.15)} \\
\bottomrule
\end{tabular}
}
\end{table}

GeoERM’s best result appears at \( r = 10 \), with an error rate of 1.02\%, while pERM (the second-best method) has an error of 1.33\%. Single-task and pooled regression perform much worse, showing they do not effectively use shared structure in this high-dimensional setting. By integrating geometric constraints directly into the optimization process, GeoERM captures subtle relationships between tasks and remains robust as model complexity (rank \( r \)) changes.

These real-data results reinforce earlier simulation findings, validating the practical advantages of geometry-aware multi-task learning in real-world scenarios.

\section{Discussions}\label{discussion}

We introduced GeoERM, a geometric framework for multi-task learning that explicitly integrates the intrinsic geometry of representation matrices into the optimization process. By combining Riemannian gradient computation and retraction operators, our method ensured that parameter updates remain on the manifold at every step, preserving orthonormality and leveraging manifold structure for more stable learning. Our theoretical analysis established the geometric foundations of this approach, showing how tangent space projections and manifold retractions enable effective, geometry-aware optimization.

Extensive numerical experiments showed that GeoERM consistently outperforms existing methods across diverse scenarios. Simulations confirmed GeoERM’s adaptability to changing sample sizes (\( n \)), feature dimensions (\( p \)), and task numbers (\( T \)), especially when \( n \gtrsim p + \log T \). Applying GeoERM to the Human Activity Recognition dataset further supported these results, yielding much lower classification error rates than conventional multi-task learning approaches. Together, these findings demonstrate the practical benefits of incorporating geometric constraints into representation learning for multi-task problems.

\paragraph{Limitations and Future Directions}  
Despite these advances, several theoretical and practical challenges remain. Theoretically, deriving rigorous error bounds in a Riemannian setting is non-trivial. Classic results in high-dimensional Euclidean statistics \citep{wainwright2019high} depend on Euclidean metrics and linear assumptions. In contrast, estimation on a Riemannian manifold involves curved spaces and different notions of distance, requiring adaptations of standard tools like concentration inequalities for sample covariance matrices. An important open direction is to develop Riemannian concentration inequalities.

Practically, our reliance on a factorization \(\beta = A \theta\), while grounded in prior work \citep{wang2018learning, bastani2021predicting, chua2021fine, li2022transfer, tian2023learning, duan2023adaptive, gu2023commute}, may limit scalability in large-scale or high-dimensional settings. Nonparametric techniques, such as diffusion maps or RKHS-based approaches, could offer more flexible ways to discover low-dimensional geometric structures without strict parametric assumptions. Integrating these methods with geometry-aware optimization could create more scalable frameworks suitable for complex datasets, including those found in biomedical imaging or deep learning contexts, where low-dimensional manifolds often underlie observed data.

A promising path forward is twofold. First, establish error bounds for single-task learning on Riemannian manifolds and then extend these results to multi-task problems. This would allow us to do rigorous uncertainty checks. Second, develop more flexible, nonparametric frameworks that preserve the benefits of geometric optimization while addressing scalability concerns in large, high-dimensional problems. Strengthening both the theoretical and practical aspects of geometric MTL would significantly broaden its impact in modern machine learning.

\begingroup
\spacingset{1}
\bibliographystyle{jasa3}
\bibliography{main}
\endgroup

\newpage

\appendix
\section*{Appendices}
\setcounter{page}{1}
\section{Supplementary Experiments} \label{Supplementary Experiments}
Figure~\ref{fig:hpc505} repeats the experiment with \( p = 50 \), increasing feature dimensionality. As expected, the task becomes more challenging, and performance differences between methods become more pronounced. GeoERM again achieves the lowest maximum error, demonstrating strong adaptability under higher dimensionality. Notably, pERM’s performance is closer to GeoERM’s with outliers, indicating stable behavior under these scenarios. Spectral methods follow closely but have slightly higher errors.
\begin{figure}[ht]
    \centering
    \includegraphics[width=0.97\linewidth]{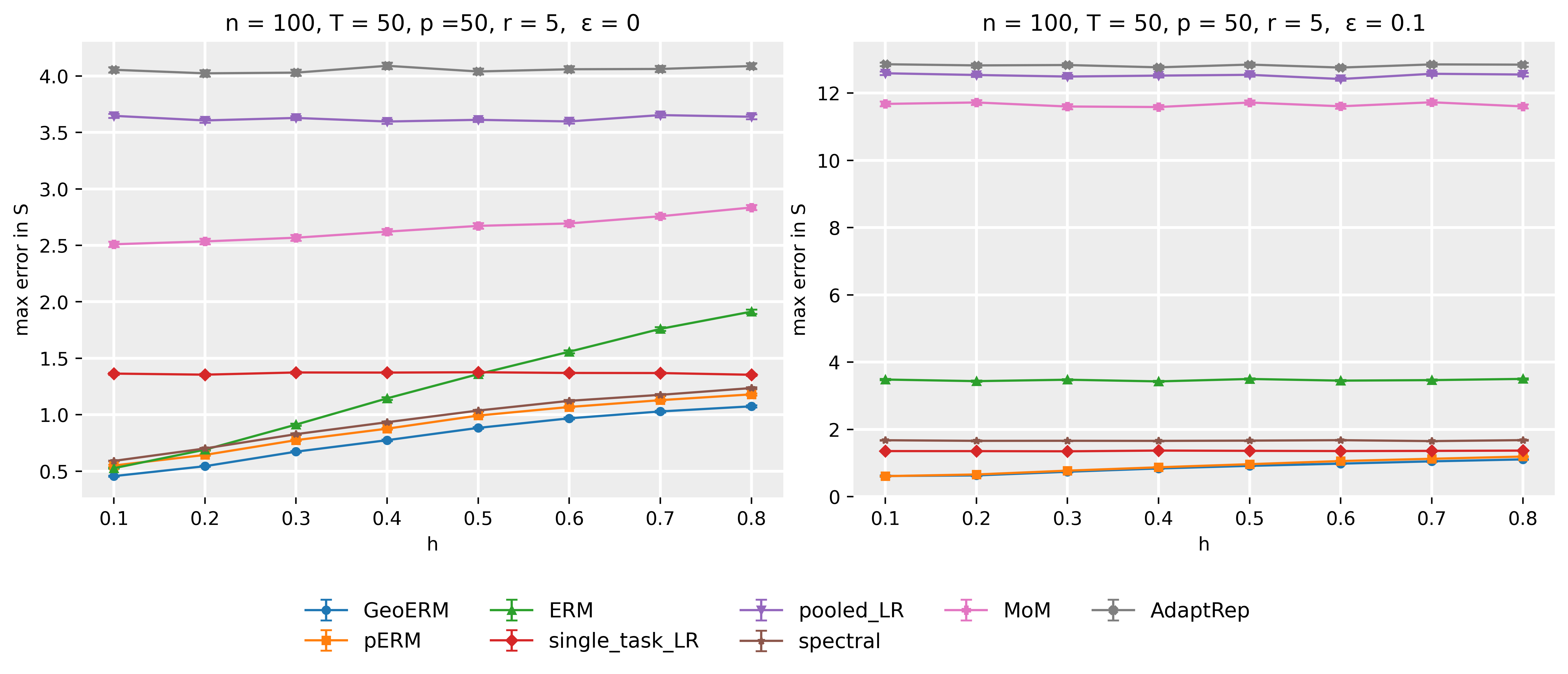}
    \caption{Maximum error across varying \( h \), under \(\epsilon = 0\) (left) and \(\epsilon = 0.1\) (right). Simulations: \( n = 100 \), \( T = 50 \), \textbf{\( p = 50 \)}, \( r = 5 \). Evaluation metrics and computational settings are in Sections~\ref{evalmetrics}.}
    \label{fig:hpc505}
\end{figure}

Figure~\ref{fig:hpc805} increases dimensionality further to \(p = 80\). GeoERM continues to perform best under \(\epsilon = 0\). However, for outlier tasks (\(\epsilon = 0.1\)), GeoERM’s performance declines. This drop matches the theoretical condition from \citet{tian2023learning} that \( n \gtrsim p + \log T \) is needed. Nevertheless, GeoERM remains the best performer on non-outlier tasks due to the weaker requirement \( n \gtrsim r + \log T \). The sharper decline in performance with increasing \( h \) reflects the difficulty of learning under high-dimensional heterogeneity.
\begin{figure}[H]
    \centering
    \includegraphics[width=0.97\linewidth]{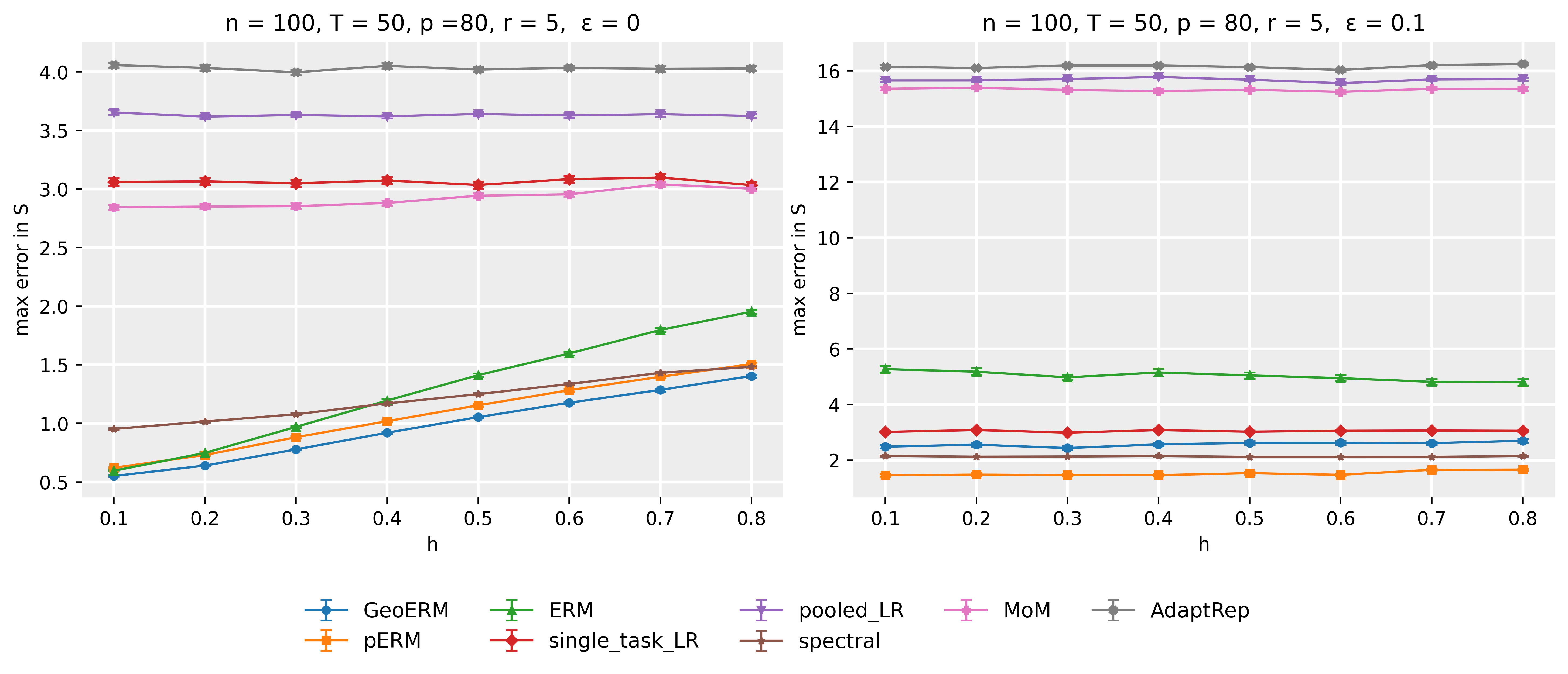}
    \caption{Maximum error across varying \( h \), under \(\epsilon = 0\) (left) and \(\epsilon = 0.1\) (right). Simulations: \( n = 100 \), \( T = 50 \), \textbf{\( p = 80 \)}, \( r = 5 \).}
    \label{fig:hpc805}
\end{figure}

To restore the balance \(n \gtrsim p + \log T\), we increase the sample size to \(n = 150\). As shown in Figure~\ref{fig:hpc150805}, GeoERM again outperforms all other methods, resembling the pattern seen at \(p = 50\). pERM and spectral methods rank second and third. ERM performs well at low \(h\) but deteriorates around \(h = 0.5\), showing its limitations under more heterogeneous conditions.

\begin{figure}[H]
    \centering
    \includegraphics[width=0.97\linewidth]{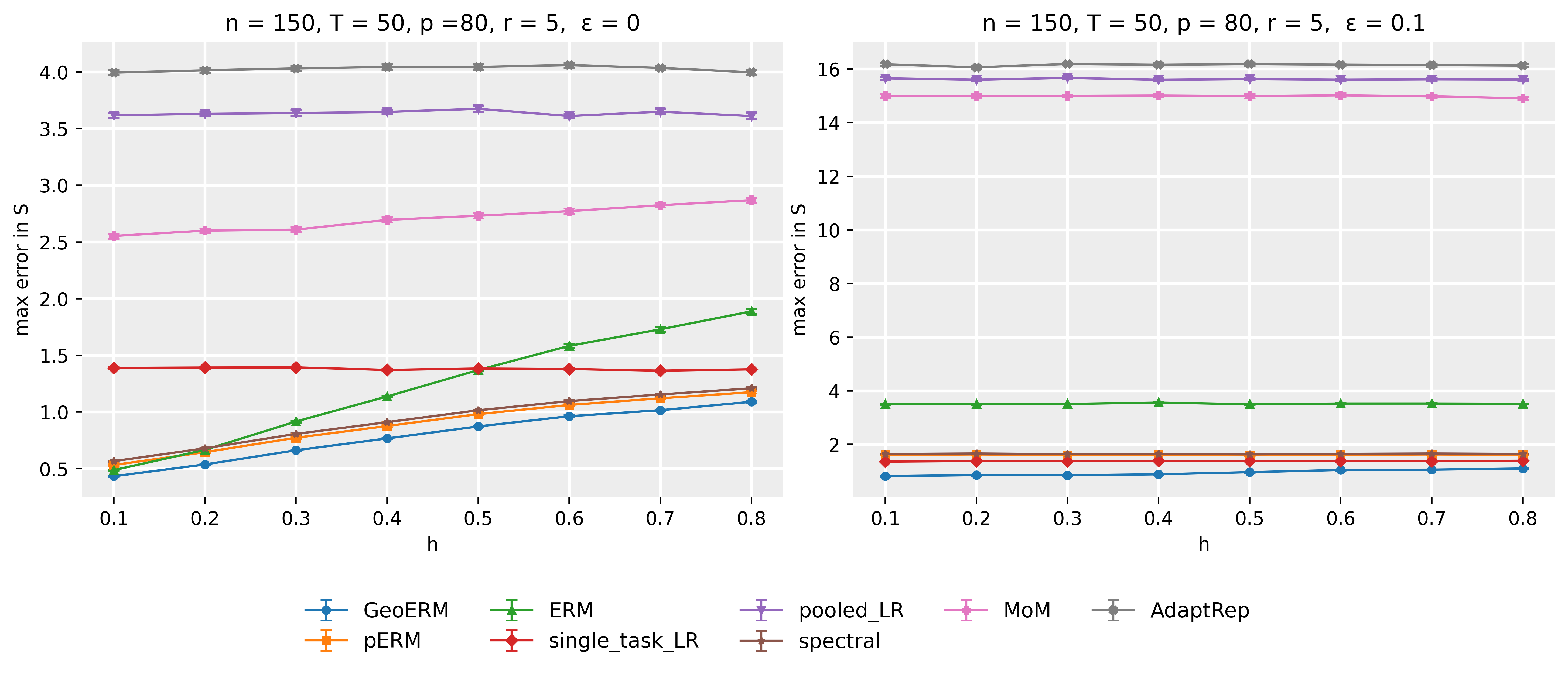}
    \caption{Maximum error across varying \( h \), under \( \epsilon = 0 \) (left) and \( \epsilon = 0.1 \) (right). Simulations: \( n = 150 \), \( T = 50 \), \( p = 80 \), \( r = 5 \).}

    \label{fig:hpc150805}
\end{figure}

\subsection{Simulation with Different Sample Size n}
We next vary \(n \in [60, 200]\) with fixed \(h = 0.5\).  As shown in Figure~\ref{fig:hpc505n}, single-task regression is highly sensitive to small \( n \) due to a lack of cross-task information sharing. In contrast, GeoERM, pERM, and spectral methods remain stable across sample sizes, reflecting their ability to exploit shared structures. GeoERM achieves the lowest maximum error across all \( n \), with or without outliers. As \( n \) increases, the gap between single-task and multitask methods narrows, and pERM and spectral methods improve as more data becomes available.
\begin{figure}[H]
    \centering
    \includegraphics[width=1\linewidth]{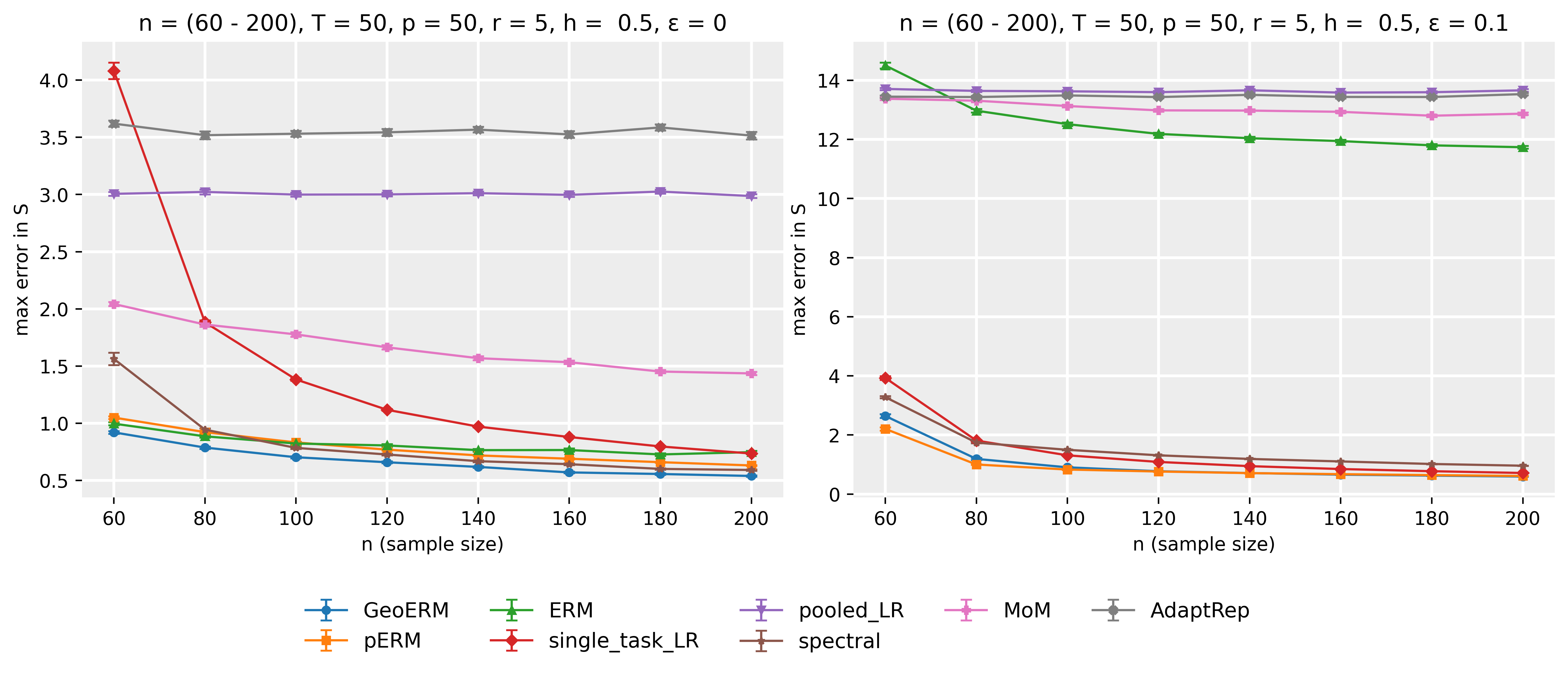}
    \caption{Simulation with varying sample size \( n \). Simulation with varying sample size \( n \in [60, 200] \) at fixed \( h = 0.5 \), \( T = 50 \), \( p = 50 \), and \( r = 5 \). Left: \( \epsilon = 0 \); Right: \( \epsilon = 0.1 \). See Section 3.2 for metric definitions and additional details.}
    \label{fig:hpc505n}
\end{figure}

\subsection{Simulation with Different Task Number T}

Finally, we vary \(T \in [20, 100]\) with \(h = 0.5\). Figure~\ref{fig:hpc505T} shows that GeoERM, pERM, and spectral methods benefit from increased task numbers, especially at smaller \( T \), consistent with observations in \citet{tian2023learning}. MoM improves more slowly and requires larger \( T \) to catch up. GeoERM consistently achieves the lowest error. Even in the presence of outliers (\(\epsilon = 0.1\)), GeoERM, pERM, and spectral methods remain robust, whereas others suffer from negative transfer.
\begin{figure}[ht]
    \centering
    \includegraphics[width=1\linewidth]{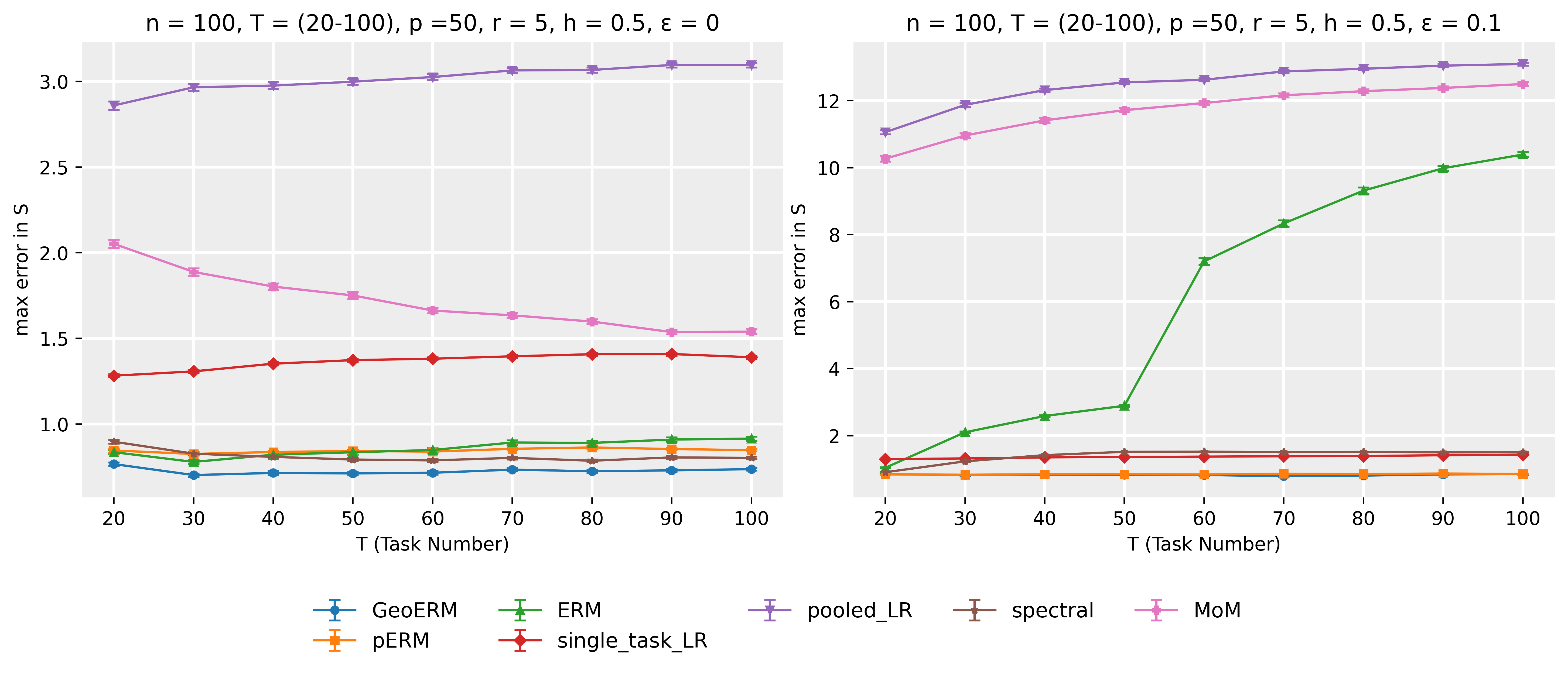}
    \caption{Maximum error across varying task number \( T \), under \(\epsilon = 0\) (left) and \(\epsilon = 0.1\) (right). Simulations: \( n = 100 \), \( p = 50 \), \( r = 5 \), \( h = 0.5 \).}
    \label{fig:hpc505T}
\end{figure}

\section{Related Works}\label{Related Works}

\paragraph{Theoretical Foundations and Efficiency Gains in MTL}
The theoretical foundations of multi-task learning (MTL) are well-established. Early work by \citet{baxter2000model} showed that when tasks derive from a common distribution, sharing representations can yield more efficient learning. Later, \citet{maurer2016benefit} introduced complexity-based analyses to quantify these gains, providing a more rigorous understanding of MTL’s benefits. Subsequent research considered the impact of task diversity and structural assumptions. For example, \citet{du2020few}, \citet{tripuraneni2020theory}, and \citet{tripuraneni2020provable} demonstrated how factors such as non-linear structure and heterogeneity shape sample complexity and generalization performance. On the optimization side, \citet{thekumparampil2021statistically} developed an alternating gradient descent algorithm that matches the effectiveness of standard empirical risk minimization while alleviating certain non-convex difficulties.

\paragraph{Applications and Specialized Frameworks}
MTL’s strengths extend across many application domains. In federated learning, it integrates information from distributed datasets while safeguarding data privacy \citep{collins2021exploiting, duchi2022subspace}. Other specialized formulations apply MTL principles to tensor representation meta-learning \citep{deng2022learning}, conditional meta-learning \citep{denevi2020advantage}, and matrix completion \citep{zhou2021multi}. When tasks share similar underlying supports, as in sparse or structured parameter scenarios, MTL improves both data efficiency and predictive accuracy \citep{xu2021multitask, li2023targeting}.

\paragraph{Addressing Heterogeneity and Outliers}
A key challenge in MTL lies in handling heterogeneous tasks and data contamination. Traditional approaches often assume that all tasks adhere to a single representation structure, an assumption easily broken by variations in data distribution \citep{zhang2021survey} or by adversarial contamination \citep{qiao2017learning, qiao2018outliers}. More recent frameworks relax these assumptions, permitting richer task-specific representations that can adapt to diverse structural and environmental conditions.

\paragraph{Distance- and Similarity-Based Frameworks}
Early MTL and transfer learning methods frequently measured task similarity using simple Euclidean or \(\ell_1\)-norm distances, adapting them to a range of settings, including high-dimensional generalized linear models \citep{tian2023transfer}, graphical models \citep{li2022transfer}, semi-supervised classification \citep{zhou2024doubly}, unsupervised Gaussian mixture modeling \citep{tian2022unsupervised, tian2024towards}, and differential privacy \citep{li2024federated}. Such approaches rely on the premise that tasks cluster together in a well-defined metric space \citep{bastani2021predicting, li2022transfer, duan2023adaptive, gu2023commute}. Beyond straightforward distance metrics, angle-based measures refine the notion of similarity. For instance, \citet{gu2022robust} probed alignment by examining angles between parameter vectors. Smaller angles indicate tighter relationships and thus more effective knowledge sharing. Other frameworks, like \citet{tian2023learning}, anchor tasks relative to a central structure, enabling a balance between global coherence and local adaptability.

\paragraph{Manifold-Based Approaches to MTL}
An emerging body of work embeds task parameters on manifolds, using geometric structures to enhance robustness. Traditional manifold-based MTL approaches often treat orthogonality and manifold constraints as static regularizers rather than as integral components of the optimization process \citep{ishibashi2022multi, xiao2019manifold, luo2012manifold, jie2015manifold}. Although these methods acknowledge that parameters may reside on low-dimensional manifolds, improving performance in areas like image classification, disease detection, and face spoofing, they rarely exploit the full power of Riemannian geometry. Recent studies, such as \citet{zheng2023sofari}, have begun imposing orthogonality conditions reminiscent of Stiefel manifold structures. However, these methods emphasize bias correction and parameter estimation, often imposing geometric constraints only as post-hoc adjustments. Our method instead embeds Riemannian geometry into the core of the MTL framework as a structural design. This geometry-aware structure enhances both stability and robustness under heterogeneous or adversarial conditions, as demonstrated throughout our analysis.

\section{Preliminaries on Riemannian Optimization}\label{Preliminaries on Riemannian Optimization}
\subsection{Stiefel Manifold}
The matrices \(\boldsymbol{A}^{(t)}\) appearing in the representation layer satisfy the orthogonality condition \(\boldsymbol{A}^{(t)^\top}\boldsymbol{A}^{(t)} = \boldsymbol{I}_r\), which places them on the Stiefel manifold. Formally, the Stiefel manifold \(\operatorname{St}(p,r)\) is defined as
\[
\operatorname{St}(p, r) \;=\; \{\boldsymbol{A}\in \mathbb{R}^{p\times r} : \boldsymbol{A}^\top \boldsymbol{A}=\boldsymbol{I}_r\},
\]
with \(r \leq p\). It is the set of all \(p\times r\) orthonormal matrices, representing \(r\) orthonormal vectors in \(\mathbb{R}^p\). By working on \(\operatorname{St}(p,r)\), one preserves the orthogonality structure of the representation matrices \(\boldsymbol{A}^{(t)}\) throughout the estimation process, ensuring that each task's representation is both geometrically sound and efficiently structured.

\begin{proposition}[Stiefel manifold structure]\label{prop:StiefelStructure}
Define \(h: \mathbb{R}^{p\times r} \to \operatorname{Sym}(r)\) by \(h(\boldsymbol{A}) = \boldsymbol{A}^\top \boldsymbol{A}-\boldsymbol{I}_r\). Then \(h\) is a defining function for \(\operatorname{St}(p,r)\), making \(\operatorname{St}(p,r)\) an embedded Riemannian submanifold of \(\mathbb{R}^{p\times r}\). Moreover, the dimension of \(\operatorname{St}(p,r)\) is 
\[
\dim(\operatorname{St}(p,r)) \;=\; pr - \tfrac{r(r+1)}{2}.
\]
\end{proposition}

\begin{proof}
To characterize the tangent and normal spaces of \(\operatorname{St}(p,r)\), consider \(\boldsymbol{A}\in\operatorname{St}(p,r)\) and an arbitrary \(\boldsymbol{Z}\in\mathbb{R}^{p\times r}\). The tangent space \(T_{\boldsymbol{A}}\operatorname{St}(p,r)\) is given by all \(\boldsymbol{H}\) such that \(\boldsymbol{A}^\top \boldsymbol{H} + \boldsymbol{H}^\top \boldsymbol{A}=0\), while the normal space \(N_{\boldsymbol{A}}\operatorname{St}(p,r)\) consists of all matrices of the form \(\boldsymbol{A}\boldsymbol{S}\) with \(\boldsymbol{S}\in\operatorname{Sym}(r)\). To identify the orthogonal projection of \(\boldsymbol{Z}\) onto the tangent space, we impose two conditions: first, \(\boldsymbol{Z}-\mathcal{P}_{\boldsymbol{A}}(\boldsymbol{Z})\in N_{\boldsymbol{A}}\operatorname{St}(p,r)\), so that there exists \(\boldsymbol{S}\in\operatorname{Sym}(r)\) with \(\boldsymbol{Z}-\mathcal{P}_{\boldsymbol{A}}(\boldsymbol{Z})=\boldsymbol{A}\boldsymbol{S}\); second, \(\mathcal{P}_{\boldsymbol{A}}(\boldsymbol{Z})\in T_{\boldsymbol{A}}\operatorname{St}(p,r)\), hence \(\boldsymbol{A}^\top \mathcal{P}_{\boldsymbol{A}}(\boldsymbol{Z})+ \mathcal{P}_{\boldsymbol{A}}(\boldsymbol{Z})^\top \boldsymbol{A}=0\). Substituting \(\mathcal{P}_{\boldsymbol{A}}(\boldsymbol{Z})=\boldsymbol{Z}-\boldsymbol{A}\boldsymbol{S}\) and rearranging these conditions leads to \(\boldsymbol{S}=\operatorname{sym}(\boldsymbol{A}^\top \boldsymbol{Z})\), thereby showing that the tangent space and normal space decomposition is consistent and that \(\operatorname{St}(p,r)\) is a well-defined embedded submanifold. The dimension formula follows from counting constraints imposed by \(\boldsymbol{A}^\top \boldsymbol{A}=\boldsymbol{I}_r\), which removes \(\frac{r(r+1)}{2}\) degrees of freedom from the original \(pr\)-dimensional ambient space.
\end{proof}

\subsection{Skew-Symmetric Matrices}\label{Skew-Symmetric Matrices}

A skew-symmetric matrix \(\boldsymbol{\Omega} \in \mathbb{R}^{r \times r}\) satisfies \(\boldsymbol{\Omega}^\top = -\boldsymbol{\Omega}\). This property forces all diagonal elements to be zero.

The set \(\operatorname{Skew}(r)\) of all \(r \times r\) skew-symmetric matrices forms a real vector space and has dimension \(\tfrac{r(r-1)}{2}\). It is the Lie algebra of the special orthogonal group \(SO(r)\), capturing the directions of infinitesimal rotations.

For example, when \(r=3\), a general element of \(\operatorname{Skew}(3)\) is:
\[
\boldsymbol{\Omega} = \begin{pmatrix}
0 & -\alpha & -\beta \\
\alpha & 0 & -\gamma \\
\beta & \gamma & 0
\end{pmatrix}
\]
with \(\alpha, \beta, \gamma \in \mathbb{R}\). This structure characterizes the tangent directions associated with orthogonal constraints, as seen in optimization problems on the Stiefel manifold.

\subsection{Riemannian Gradient and Metric Equivalence on the Stiefel Manifold}\label{Riemannian Gradient and Metric Equivalence on the Stiefel Manifold}

On the Stiefel manifold  \(\operatorname{St}(p, r)\), the Riemannian gradient indicates how to descend the objective while remaining on the manifold. We obtain it by projecting the Euclidean gradient onto the tangent space. 


Let \( f \colon \operatorname{St}(p, r) \to \mathbb{R} \) be a smooth function. Because \( f \) is smooth, we can extend it to a neighborhood of \(\operatorname{St}(p, r)\) in \(\mathbb{R}^{p \times r}\), obtaining \(\bar{f} \colon \mathbb{R}^{p \times r} \to \mathbb{R}\). For a point \( x \in \operatorname{St}(p, r) \), the Riemannian gradient \( \tilde{\nabla} f(x) \) at \( x \) must satisfy  
\( \forall v \in T_x \operatorname{St}(p, r), \quad \langle v, \tilde{\nabla} f(x) \rangle_x = Df(x)[v] = \langle v, \nabla \bar{f}(x) \rangle \),  
where \( \langle \cdot,\cdot \rangle_x \) is the Riemannian metric at \( x \), and \( Df(x)[v] \) is the directional derivative of \( f \) at \( x \) along \( v \).

Since \( T_x \operatorname{St}(p, r) \) is a subspace of \( \mathbb{R}^{p \times r} \), we can break down the Euclidean gradient \( \nabla \bar{f}(x) \) into parts parallel and perpendicular to the tangent space:  
\( \nabla \bar{f}(x) = \nabla \bar{f}(x)_\parallel + \nabla \bar{f}(x)_\perp \),  
where \( \nabla \bar{f}(x)_\parallel \in T_x \operatorname{St}(p, r) \) and \( \nabla \bar{f}(x)_\perp \in N_x \operatorname{St}(p, r) \), the normal space at \( x \). Because the normal component does not contribute to directional derivatives along the manifold, we have  
\( \langle v, \tilde{\nabla} f(x) \rangle_x = \langle v, \nabla \bar{f}(x)_\parallel \rangle \).

On \( \operatorname{St}(p, r) \), the Riemannian metric comes directly from the usual inner product in \( \mathbb{R}^{p \times r} \). As a result, projecting the Euclidean gradient onto the tangent space \( T_x \operatorname{St}(p, r) \) is straightforward:
\begin{equation}
\mathcal{P}_{T_x \operatorname{St}(p, r)}(\nabla \bar{f}(x)) = \nabla \bar{f}(x) - N_x \operatorname{St}(p, r).
\label{eq:projection_tangent_space}
\end{equation}

\begin{remark}
For a general Riemannian manifold \( \mathcal{M} \), the Riemannian gradient \( \tilde{\nabla} f(x) \) depends on the chosen metric and may differ greatly from the simple projection of \( \nabla \bar{f}(x) \). In our case, where \( \mathcal{M}=\operatorname{St}(p,r) \) is a submanifold of \( \mathbb{R}^{p \times r} \) with the standard Euclidean metric, the Riemannian gradient is just:
\(
\tilde{\nabla} f(x) = \mathcal{P}_{T_x \operatorname{St}(p, r)}(\nabla \bar{f}(x)).
\)
This convenient relationship simplifies computing Riemannian gradients on the Stiefel manifold, making geometry-based optimization methods more accessible.
\end{remark}

\subsection{Retraction}

\subsubsection{Connection to SVD and Gram Matrix}
The polar retraction can be understood through the singular value decomposition (SVD). Suppose  
\( \boldsymbol{A}^{(t)} + \boldsymbol{H} = \boldsymbol{U} \boldsymbol{\Sigma} \boldsymbol{W}^\top \),  
with \( \boldsymbol{U} \in \mathbb{R}^{p \times r} \), \( \boldsymbol{\Sigma} \in \mathbb{R}^{r \times r} \), and \( \boldsymbol{W} \in \mathbb{R}^{r \times r} \). Substituting into the polar retraction definition:  
\( (\boldsymbol{A}^{(t)} + \boldsymbol{H})^\top(\boldsymbol{A}^{(t)} + \boldsymbol{H}) = \boldsymbol{I}_r + \boldsymbol{H}^\top \boldsymbol{H} = \boldsymbol{W}\boldsymbol{\Sigma}^2\boldsymbol{W}^\top \).  
Thus,  
\( (\boldsymbol{I}_r + \boldsymbol{H}^\top \boldsymbol{H})^{-1/2} = \boldsymbol{W}\boldsymbol{\Sigma}^{-1}\boldsymbol{W}^\top \).  
Putting it all together, we get  
\( \mathcal{R}_{\boldsymbol{A}^{(t)}}(\boldsymbol{H}) = \boldsymbol{U}\boldsymbol{W}^\top \). This operation finds the closest orthonormal matrix to \( \boldsymbol{A}^{(t)} + \boldsymbol{H} \).

\begin{remark}
    The polar retraction is essentially picking out the orthonormal factor \(\boldsymbol{U}\boldsymbol{W}^\top\). By doing so, it naturally returns the updated matrix to the Stiefel manifold in a stable way.
\end{remark}

\begin{remark}
    Viewing the polar retraction through the SVD shows how straightforward it is to implement. Turning \(\boldsymbol{A}^{(t)} + \boldsymbol{H}\) into \(\boldsymbol{U}\boldsymbol{W}^\top\) directly enforces orthonormality, making the method robust and practical.
\end{remark}

Proposition \ref{well-defined retraction} confirms that \(\mathcal{R}_{\boldsymbol{A}^{(t)}}\) is a well-defined smooth retraction. Each update step returns smoothly to the manifold. Proposition \ref{Uniqueness of polar retraction} confirms that the retraction \(\mathcal{R}_{\boldsymbol{A}^{(t)}}\) is unique.

\begin{proposition}[Well-defined retraction]\label{well-defined retraction}
Let \(\mathcal{M} = \operatorname{St}(p,r)\). Define \(\mathcal{R}: T\mathcal{M} \to \mathcal{M}\) by
\[
\mathcal{R}_{\boldsymbol{A}}(\boldsymbol{H}) \;=\; (\boldsymbol{A} + \boldsymbol{H})(\boldsymbol{I}_r + \boldsymbol{H}^\top \boldsymbol{H})^{-1/2}.
\]
Then \(\mathcal{R}\) is a smooth and well-defined retraction on \(\mathcal{M}\).\\

\begin{proof}
By construction, \(\mathcal{R}_{\boldsymbol{A}}(\boldsymbol{0}) = \boldsymbol{A}\), ensuring the first property of a retraction. To verify that \(\mathcal{R}_{\boldsymbol{A}}(\boldsymbol{H}) \in \mathcal{M}\) for all \(\boldsymbol{H} \in T_{\boldsymbol{A}}\mathcal{M}\), observe that
\[
(\mathcal{R}_{\boldsymbol{A}}(\boldsymbol{H}))^\top \mathcal{R}_{\boldsymbol{A}}(\boldsymbol{H}) 
= (\boldsymbol{I}_r + \boldsymbol{H}^\top\boldsymbol{H})^{-1/2}(\boldsymbol{I}_r + \boldsymbol{H}^\top\boldsymbol{H})(\boldsymbol{I}_r + \boldsymbol{H}^\top\boldsymbol{H})^{-1/2}
= \boldsymbol{I}_r.
\]
Thus, \(\mathcal{R}_{\boldsymbol{A}}(\boldsymbol{H}) \in \operatorname{St}(p,r)\).

Next, consider smoothness. The map \((\boldsymbol{A}, \boldsymbol{H}) \mapsto (\boldsymbol{A}+\boldsymbol{H})(\boldsymbol{I}_r+\boldsymbol{H}^\top\boldsymbol{H})^{-1/2}\) is a composition of smooth functions: matrix addition, inversion, and the matrix inverse square root of a positive-definite matrix are smooth operations on their respective domains. Hence \(\mathcal{R}\) is smooth.

To show that the differential at \(\boldsymbol{H}=\boldsymbol{0}\) is the identity on \(T_{\boldsymbol{A}}\mathcal{M}\), we write
\[
\frac{d}{dt}\mathcal{R}_{\boldsymbol{A}}(t\boldsymbol{H})\bigg|_{t=0} 
= \frac{d}{dt}\big[(\boldsymbol{A}+t\boldsymbol{H})(\boldsymbol{I}_r + t^2\boldsymbol{H}^\top\boldsymbol{H})^{-1/2}\big]\bigg|_{t=0}.
\]

Set \(\gamma(t) = \boldsymbol{I}_r + t^2\boldsymbol{H}^\top\boldsymbol{H}\). We then consider the function \(h: P_r \to P_r\), defined by \(\boldsymbol{B} \mapsto \boldsymbol{B}^{-1/2}\), where \(P_r\) is the set of positive-definite \(r\times r\) matrices. Applying the chain rule:
\[
\frac{d}{dt}h(\gamma(t))\bigg|_{t=0} = D h(\boldsymbol{I}_r)[2t\boldsymbol{H}^\top\boldsymbol{H}]_{t=0} = 0,
\]
since the factor \(2t\) vanishes at \(t=0\).

Hence
\[
\frac{d}{dt}\mathcal{R}_{\boldsymbol{A}}(t\boldsymbol{H})\bigg|_{t=0} = \boldsymbol{H} + \boldsymbol{A}\cdot 0 = \boldsymbol{H}.
\]

This shows that \(\mathcal{R}\) is indeed a retraction: it maps \(\boldsymbol{0}\) back to \(\boldsymbol{A}\) and its differential at \(\boldsymbol{0}\) is the identity on \(T_{\boldsymbol{A}}\mathcal{M}\). Thus, \(\mathcal{R}\) is a smooth and well-defined retraction on \(\operatorname{St}(p,r)\).\\ 
\end{proof}
\end{proposition}

\begin{proposition}[Uniqueness of the polar retraction]\label{Uniqueness of polar retraction}
Let \(\mathcal{M} = \operatorname{St}(p,r)\). Consider \((\boldsymbol{A},\boldsymbol{H}) \in T\mathcal{M}\) and form the thin singular value decomposition \((\boldsymbol{A} + \boldsymbol{H}) = \boldsymbol{U}\boldsymbol{\Sigma}\boldsymbol{W}^\top\), where \(\boldsymbol{U} \in \mathcal{M}\), \(\boldsymbol{W} \in O(r)\), and \(\boldsymbol{\Sigma}\) is diagonal with strictly positive entries. Then the matrix \(\boldsymbol{U}\boldsymbol{W}^\top\) is the unique metric projection of \(\boldsymbol{A}+\boldsymbol{H}\) onto \(\mathcal{M}\). In particular, the polar retraction defined by \(\mathcal{R}_{\boldsymbol{A}}(\boldsymbol{H}) = \boldsymbol{U}\boldsymbol{W}^\top\) is unique.\\

\begin{proof}
Fix \(\boldsymbol{A}\in \mathcal{M}\) and \(\boldsymbol{H}\in T_{\boldsymbol{A}}\mathcal{M}\). By definition, we have \(\boldsymbol{A}+\boldsymbol{H} = \boldsymbol{U}\boldsymbol{\Sigma}\boldsymbol{W}^\top\). We aim to find
\[
\min_{\boldsymbol{Y}\in\mathcal{M}}\|\boldsymbol{A}+\boldsymbol{H}-\boldsymbol{Y}\|_F^2.
\]
Since the Frobenius norm is unitarily invariant and \(\boldsymbol{W}\in \mathcal{O}(r)\), the transformation \(\boldsymbol{Y} \mapsto \boldsymbol{Y}\boldsymbol{W}\) is bijective on \(\mathcal{M}\). Let \(\boldsymbol{Z} = \boldsymbol{Y}\boldsymbol{W}\). Then
\[
\inf_{\boldsymbol{Y}\in\mathcal{M}}\|\boldsymbol{A}+\boldsymbol{H}-\boldsymbol{Y}\|_F^2
= \inf_{\boldsymbol{Z}\in\mathcal{M}}\|\boldsymbol{U}\boldsymbol{\Sigma}-\boldsymbol{Z}\|_F^2.
\]

Write \(\boldsymbol{U}\boldsymbol{\Sigma} = [\sigma_1\boldsymbol{u}_1 \;\; \sigma_2\boldsymbol{u}_2 \;\; \cdots \;\; \sigma_r\boldsymbol{u}_r]\), where \(\{\boldsymbol{u}_i\}_{i=1}^r\) are orthonormal vectors in \(\mathbb{R}^{p}\) and \(\sigma_i>0\). Any \(\boldsymbol{Z}\in\mathcal{M}\) has orthonormal columns \(\{\boldsymbol{z}_i\}_{i=1}^r\). Thus
\[
\|\boldsymbol{U}\boldsymbol{\Sigma}-\boldsymbol{Z}\|_F^2 = \sum_{i=1}^r \|\sigma_i \boldsymbol{u}_i - \boldsymbol{z}_i\|_2^2 
= \sum_{i=1}^r (\sigma_i^2 - 2\sigma_i\langle \boldsymbol{u}_i,\boldsymbol{z}_i\rangle + 1).
\]

By the Cauchy–Schwarz inequality, \(\langle \boldsymbol{u}_i,\boldsymbol{z}_i\rangle \leq 1\), with equality if and only if \(\boldsymbol{z}_i=\boldsymbol{u}_i\). Since each \(\sigma_i>0\), minimizing the quadratic expression in \(\boldsymbol{z}_i\) forces \(\langle \boldsymbol{u}_i,\boldsymbol{z}_i\rangle=1\). Hence \(\boldsymbol{z}_i=\boldsymbol{u}_i\) for all \(i\), and we obtain the unique minimizer \(\boldsymbol{Z}=\boldsymbol{U}\). Retracing the substitution \(\boldsymbol{Z}=\boldsymbol{Y}\boldsymbol{W}\) gives \(\boldsymbol{Y}=\boldsymbol{U}\boldsymbol{W}^\top\).

This establishes that \(\boldsymbol{U}\boldsymbol{W}^\top\) is the unique solution to the projection problem and thus the polar retraction \(\mathcal{R}_{\boldsymbol{A}}(\boldsymbol{H}) = \boldsymbol{U}\boldsymbol{W}^\top\) is uniquely defined.\\
\end{proof}
\end{proposition}

\end{document}